\documentclass[11pt,twoside]{article}
\usepackage{fancyhdr}
\usepackage[colorlinks,citecolor=blue,urlcolor=blue,linkcolor=blue,bookmarks=false,breaklinks]{hyperref}
\usepackage{amsfonts,epsfig,graphicx}
\usepackage{afterpage}
\usepackage{amsmath,amssymb,amsthm} 
\usepackage{fullpage}
\usepackage[T1]{fontenc} 
\usepackage{epsf} 
\usepackage{graphics} 
\usepackage{amsfonts,amsmath}
\usepackage[sort,authoryear]{natbib} 
\usepackage{psfrag,xspace}
\usepackage{color,etoolbox}
\usepackage{enumitem}
\usepackage{bbm}
\usepackage{algorithm,algorithmic}

\usepackage{tikz}

\usepackage{setspace}

\usepackage{mathtools}

\newtheorem{theorem}{Theorem} 
 
\newtheorem{lemma}{Lemma}

\newtheorem{corollary}{Corollary}
\newtheorem{definition}{Definition}

\newtheorem{remark}{Remark}

\newcommand{\alg}{\textsc{Shaky Prepend}\xspace}
\newcommand{\ShakyPrepend}{\textsc{Shaky Prepend}\xspace}
\newcommand{\GroupPrepend}{\textsc{Group Prepend}\xspace}
\newcommand{\Prepend}{\textsc{Prepend}\xspace}
\newcommand{\SleepingExpert}{\textsc{Sleeping Expert}\xspace}
\newcommand{\FractionalShakyPrepend}{\textsc{Fractional Shaky Prepend}\xspace}

\begin{document}

\begin{center} {\LARGE{\bf{\ShakyPrepend: A Multi-Group Learner with \\
\vspace{.2cm}
Improved Sample Complexity
}}}
\\

\vspace*{.3in}

{\large{
\begin{tabular}{ccccc}
Lujing Zhang$^{\ast}$ & Daniel Hsu$^{\dagger}$ & Sivaraman Balakrishnan$^{\ddagger}$ \\
\end{tabular}

\vspace*{.1in}
\begin{tabular}{ccc}
$^\ast$Department of Statistics and Data Science \\
Carnegie Mellon University \\
\texttt{lujingz@stat.cmu.edu} 
\end{tabular}
\vspace*{.2in}
\begin{tabular}{ccc}
$^\dagger$Computer Science Department \\
Data Science Institute \\
Columbia University \\
\texttt{djhsu@cs.columbia.edu} 
\end{tabular}

\vspace*{.2in}
\begin{tabular}{ccc}
$^\ddagger$Department of Statistics and Data Science \\
Machine Learning Department \\
Carnegie Mellon University \\
\texttt{siva@stat.cmu.edu} 
\end{tabular}
\vspace*{.2in}}}

\vspace*{.2in}

\begin{abstract}
\noindent 
Multi-group learning is a learning task that focuses on controlling predictors’ conditional losses over specified subgroups. We propose \ShakyPrepend, a method that leverages tools inspired by differential privacy to obtain improved theoretical guarantees over existing approaches. Through numerical experiments, we demonstrate that \ShakyPrepend adapts to both group structure and spatial heterogeneity. We provide practical guidance for deploying multi-group learning algorithms in real-world settings.
\end{abstract}
\end{center}

\section{Introduction}
Modern ML systems are often evaluated and deployed in settings where performance must be reliable not only on average but also on many subpopulations (or \emph{groups}) of interest.
 For example, in clinical prediction and medical imaging, strong aggregate metrics can conceal severe failures on clinically meaningful (and sometimes rare) subtypes or cohorts. 
  This issue is commonly referred to as \emph{hidden stratification}~\citep{oakdenrayner2020hidden}. 

In other high-stakes, human-facing domains, inconsistent performance across subpopulations can raise fairness, accountability, and compliance concerns. In consumer lending and credit underwriting, subgroup-level errors can translate into unfair denial or mispricing and may trigger anti-discrimination scrutiny, motivating careful multi-group evaluation and monitoring~\citep{hall2021fairlending}. In recommendation and ranking, popularity bias can systematically under-expose long-tail items and under-serve users with niche preferences, even when average engagement improves~\citep{abdollahpouri2019popbias}.

These considerations motivate the \emph{multi-group learning} framework, which seeks a single predictor that achieves small excess risk simultaneously for every group in a (potentially very large and overlapping) family. As formalized by \citet{pmlr-v139-rothblum21a}, given a collection of groups $\mathcal{G}$ and a reference hypothesis class $\mathcal{H}$, the goal is to output a predictor $f$ such that, for each $g \in \mathcal{G}$, the average loss of $f$ on examples from $g$ is comparable to that of the best group-specific reference predictor $\arg\min_{h\in\mathcal{H}} L(h \mid g)$.

One challenge in multi-group learning is statistical: the number of candidate groups can be enormous (e.g., all intersections of sensitive attributes, cohort definitions, or latent strata), and enforcing uniform guarantees across groups can incur a nontrivial sample-complexity overhead. \citet{tosh2024simplenearoptimalalgorithmshidden} give simple and near-optimal algorithms for hidden stratification and multi-group learning, clarifying the structure of solutions and establishing strong guarantees, including a randomized near-optimal procedure and a simpler deterministic alternative with weaker guarantees\footnote{We follow the terminology of \citet{haghtalab2023unifying} who refer to a \emph{deterministic} predictor as one which maps features to a distribution over labels, and a \emph{non-deterministic} or \emph{randomized} predictor as one which produces a distribution over deterministic predictors.}. Follow-up work has studied richer group structure; for example, \citet{pmlr-v235-deng24a} considered \emph{hierarchical} group families and designed an interpretable decision-tree learner with near-optimal sample complexity. Complementary research avoids explicit subgroup enumeration by providing guarantees over exponential or infinite subgroup families. Following the work of \citet{kearns2018preventing}, recent advances have generalized these guarantees to adversarial online learning \citep{garg2024oracle} and to broad multi-objective frameworks \citep{Noarov2021OnlineMM,balakrishnan2026panprediction,  haghtalab2023unifying}. Beyond explicitly feature-defined groups, \citet{dai2024learning} also studies \emph{endogenous} or latent subpopulations (e.g., mixture components) and shows that per-subpopulation guarantees can be achievable even when the underlying clusters are hard to recover.
We make the following contributions to this line of work:

\vspace{.2cm}

\noindent \textbf{1. We propose \alg, a multi-group learning algorithm with improved sample complexity and group-size dependence.} \alg improves upon the multi-group learning rate of the {\sc Prepend} algorithm of \citet{tosh2024simplenearoptimalalgorithmshidden} from $O(n^{-1/3})$ to $O(n^{-2/5})$, and yields better \emph{group-size dependence}: the excess-loss guarantee for a group scales with its empirical mass $P_n(g)$ rather than being driven by the smallest group. Many existing multi-group learning methods are inherently iterative: in each round, the algorithm selects a group to audit or update using statistics computed on the same fixed sample. This adaptivity can lead to overfitting.
Leveraging the core idea of \emph{differential privacy (DP) for adaptive data analysis}, we inject carefully scaled noise into each round to make the procedure more \emph{stable}, and obtain improved generalization guarantees.

\vspace{.2cm}

\noindent \textbf{2. We connect \alg to \textit{gradient boosting} and propose a fractional variant.}
We show that \alg can be viewed through the lens of \emph{gradient boosting}:
each iteration identifies a ``hard'' slice of the population (a group with large residual error) and applies a weak corrective update targeted to that slice.
Motivated by this view, we introduce a \emph{fractional} variant that generalizes a class of multi-group learning algorithms and achieves the same sample-complexity bound as the original procedure.

\vspace{.2cm}

\noindent \textbf{3. We study practical behavior and provide guidance for using multi-group learning.}
Through simulations, we examine two complementary notions of adaptivity exhibited by \alg:
\emph{spatial adaptivity}, which adjusts to unknown structure in the instance space, and
\emph{group adaptivity}, which automatically trades off high-variance group-specific predictors against lower-variance but less tailored alternatives.
We also provide practical guidance for hyperparameter tuning and model selection in the multi-group setting.

\subsection{Related Work}

\noindent \textbf{Multi-group learning and hidden stratification.}
Multi-group learning \citep{pmlr-v139-rothblum21a} formalizes the requirement that a single predictor generalize well across many (possibly overlapping) groups, capturing both fairness and hidden-stratification concerns. \citet{tosh2024simplenearoptimalalgorithmshidden} study the structure of solutions and give simple near-optimal algorithms, serving as the closest starting point for our work. Subsequent work considers more structured group families; notably, \citet{pmlr-v235-deng24a} extend the model to hierarchically structured groups and obtain an interpretable decision tree with near-optimal sample complexity. Our work complements these results by importing DP-for-adaptivity tools to reduce the statistical cost of adaptively selecting and auditing groups.

\vspace{.2cm} 

\noindent \textbf{Multicalibration and Multiaccuracy.}
Like multi-group learning, multicalibration and multiaccuracy are \emph{multi-objective learning} frameworks motivated by fairness considerations, in which a single predictor must simultaneously satisfy many constraints over a (potentially large and intersecting) collection of groups $\mathcal{G}$.
Multiaccuracy requires the prediction residual to be small on every group $g \in \mathcal{G}$, e.g.,
$\mathbb{E}\!\left[(f(X)-Y)g(X)\right]\le \alpha$~\citep{kim2019multiaccuracy}.
Multicalibration, introduced by~\citet{hebertjohnson2018multicalibration}, is a stricter notion that strengthens calibration to hold \emph{within} each group in $\mathcal{G}$, further conditioning on the predictor's score (and thus implicitly on the outcome). Generalized notions of multicalibration have been proposed in various works~\citep{jung2021moment,zhang2024fairriskcontrol}.
These calibration-style guarantees are conceptually distinct from multi-group learning, which targets small conditional loss on every group; they are generally not equivalent.
However, \citet{kim2019multiaccuracy} show that for suitable choices of $\mathcal{G}$, multiaccuracy can imply strong multi-group learning guarantees.

\vspace{.2cm} 

\noindent \textbf{Differential privacy for adaptive data analysis.}
 A major conceptual advance in modern generalization theory is that DP implies strong stability, which in turn yields validity under adaptive query selection. \citet{dwork15stoc} show how to preserve statistical validity in adaptive data analysis, establishing a foundational link between DP and generalization in interactive settings. Related work studies holdout reuse and generalization under adaptivity~\citep{dwork15holdout}. These ideas have also influenced practical protocols for repeated evaluation. Our analysis follows this tradition: we use DP-inspired stability to control overfitting arising from adaptively selecting groups based on previous risk estimates, thereby improving sample complexity.

\vspace{.2cm} 

\noindent \textbf{Adaptive risk estimation and leaderboards (Ladder / Shaky Ladder).}
 Repeated evaluation on a fixed test set can cause adaptive overfitting, a phenomenon studied in the ``leaderboard problem.'' \citet{blum2015ladderreliableleaderboardmachine} propose the Ladder algorithm to robustly report progress without leaking too much information about the holdout. The Shaky Ladder Algorithm of \citet{hardt2017climbingshakyladderbetter}~refines the work of \citet{blum2015ladderreliableleaderboardmachine} by using DP techniques to obtain improved leaderboard guarantee. Our setting differs as our goal is learning under multi-group constraints rather than leaderboard reporting. However, the technical obstruction is similar: adaptively chosen “queries” (groups/benchmarks) can overfit. We explicitly incorporate these insights by designing group-selection and update steps that remain stable under adaptivity.

\vspace{.2cm} 

\noindent \textbf{Boosting and functional gradient methods.}
Boosting constructs a strong predictor by iteratively adding weak learners that correct current errors. AdaBoost~\citep{freund97} and gradient boosting machines~\citep{friedman01} are canonical examples, with the latter interpreting boosting as greedy functional gradient descent on a loss functional. \alg shares this iterative corrective flavor: each iteration identifies a currently worst-performing group and performs an update concentrated on that slice, resembling a boosting step on group-conditional residuals. This viewpoint helps interpret our method and motivates certain practical improvements to the algorithm.


\subsection{Paper Outline}
In Section~\ref{sec:background}, we formalize the multi-group learning setup and notation, review some useful results on concentration bounds for the empirical conditional risk and review the differential privacy tools we use.
In Section~\ref{sec:Shaky Prepend}, we introduce \ShakyPrepend, an algorithm that achieves an improved theoretical bound for this problem compared to prior deterministic predictors. 
Finally, Section~\ref{sec:Experiment} provides simulation studies that complement our theoretical analysis and offer insights into the practical behavior of these algorithms.

\section{Background}\label{sec:background}
\subsection{Setting and Notation}
Let $\mathcal{X}$ denote the input space, $\mathcal{Y}$ the label space, $\mathcal{Z}$ the prediction (output) space, and $\mathcal{H} \subset \{h : \mathcal{X} \to \mathcal{Z}\}$ the hypothesis class. 
We define a group indicator function as $g : \mathcal{X} \to \{0,1\}$, and let $\mathcal{G}$ be the collection of groups of interest. 
Given a bounded loss function $\ell : \mathcal{Z}\times\mathcal{Y} \to [0,1]$, we denote by $L(f \mid g)$ the expected conditional loss
\[
L(f \mid g)
:= 
\mathbb{E}_{(x, y) \sim \mathcal{D}}
\left[\, \ell(f(x), y) \mid g(x) = 1 \,\right].
\]
In practice, we do not have direct access to the underlying distribution~$\mathcal{D}$; instead, we observe an i.i.d.\ sample 
$S_n = \{(x_i, y_i) : i = 1, \ldots, n\}$. 
Let $P_{S_n}(g)$ denote the empirical probability that $g(x) = 1$ in the sample. 
We define the empirical conditional loss as
\[
L_{S_n}(f \mid g)
:=
\frac{1}{\#_{S_n}(g)}
\sum_{i=1}^n 
\mathbbm{1}_{\{g(x_i)=1\}}\, \ell(f(x_i), y_i),
\]
where $\#_{S_n}(g) := \sum_{i=1}^n \mathbbm{1}_{\{g(x_i)=1\}}$ denotes the number of samples belonging to group $g$.

The multi-group learning problem seeks a predictor $f$ that simultaneously guarantees small conditional risk for every group:
\[
L(f \mid g) 
\le 
\inf_{h \in \mathcal{H}} L(h \mid g) + \epsilon_n(g),
\qquad \forall\, g \in \mathcal{G},
\]
where $\epsilon_n(g)$ is a non-increasing function of $\#_{S_n}(g)$. Equivalently, the goal is to obtain a collection of sharp 
oracle inequalities indexed by groups, comparing $f$ to the best-in-class predictor on each $g$.
For simplicity, we use $L_n(f|g)$, $P_n(g)$ and $\#_n(g)$ as shorthand.

 \subsection{Empirical Conditional Risk Convergence Result}
 \par We largely follow the notation of \citet{tosh2024simplenearoptimalalgorithmshidden}.
 For a class of $\{0,1\}$-valued functions $\mathcal{F}$ defined over a domain $\mathcal{X}$, its $k$-th shattering coefficient, is given by
 \[\Pi_k(\mathcal{F})=\max_{x_1,...,x_k\in\mathcal{X}}|\{(f(x_1),...,f(x_k)):f\in\mathcal{F}\}|.\]
The thresholded class associated with a real-valued function class $\mathcal{F}$ is defined as
 \[\mathcal{F}_{\text{thresh}}=\{x\to\mathbbm{1}[f(x)>\tau]:f\in\mathcal{F},\tau\in\mathbb{R}\}.\]
 Given a hypothesis class $\mathcal{H}$ and a loss function $\ell:\mathcal{Z}\times\mathcal{Y}\to [0,1]$, the loss-composed class is 
 \[\ell \circ \mathcal{H}=\{(x,y)\to\ell(h(x),y):h\in\mathcal{H}\}.\]
With these definitions in place, the following theorem provides a high-probability uniform convergence bound, bridging empirical conditional risk and the true conditional risk:
 \begin{theorem}[\citealp*{tosh2024simplenearoptimalalgorithmshidden}]\label{thm:Cond_Emp_Risk}
Let $\mathcal{H}$ be a hypothesis class, let $\mathcal{G}$ be a set of groups, and let $\ell: \mathcal{Z} \times \mathcal{Y} \to [0,1]$ be a loss function. With probability at least $1 - \delta$,
\[
\left| L(h \mid g) - L_n(h \mid g) \right| \le 9 \sqrt{\frac{D}{\#_n(g)}}, \quad \forall (h, g) \in \mathcal{H} \times \mathcal{G},
\]
where
\(
D = 2 \ln\left( \Pi_{2n}\big( (\ell \circ \mathcal{H})_{\mathrm{thresh}} \big) \Pi_{2n}(\mathcal{G}) \right) + \ln\left( \frac{8}{\delta} \right).
\)
\end{theorem}
\subsection{Differential Privacy}
\begin{definition}[Differential Privacy]
A randomized algorithm $\mathcal{M}$ with domain $\mathcal{D}$ and range $\mathcal{R}$ 
is said to satisfy $(\varepsilon,\delta)$-\emph{differential privacy} if for any two 
adjacent datasets $D, D' \in \mathcal{D}$ differing in at most one element, and for 
any subset of outputs $S \subseteq \mathcal{R}$, we have
\[
    \Pr[\mathcal{M}(D) \in S] 
    \le e^{\varepsilon} \Pr[\mathcal{M}(D') \in S] + \delta.
\]
\end{definition}

\begin{definition}[$\ell_1$-sensitivity]\label{def:l1-sensitivity}
Let $\mathcal{X}$ be a data universe and let $\mathcal{S}$ be a family of datasets over $\mathcal{X}$. Denote $S\triangle S'$ as the symmetric difference.
For $f:\mathcal{S}\to\mathbb{R}^k$, its $\ell_1$-sensitivity is
\[
\Delta f \;:=\; \max_{\substack{S,S'\in\mathcal{S}\\ |S\triangle S'|=1}}
\bigl\|f(S)-f(S')\bigr\|_1,
\]
\end{definition}
\noindent Equivalently, $|S\triangle S'|=1$ means $S$ and $S'$ differ by one record, so $\Delta f$ is the maximum change in $f$ under a one-record change~\citep{dwork2006calibrating,dwork2014algorithmic}.

Our algorithm is motivated by the Sparse Vector Technique~\citep{DworkNRRV09} and we use the version in~\citet{dwork2014algorithmic}. We use a generalized variant for \(\Delta\)-sensitive queries and a flexible stopping rule (Algorithm~\ref{alg:adjusted_sparse}). 
Intuitively, each answered query leaks some information about the dataset, so under naive composition the privacy cost grows roughly with the number of queries. 
With the Sparse Vector Technique (SVT), however, the privacy loss depends primarily on the number of \emph{updates} (i.e., threshold crossings), rather than the total number of queries. 
In the classical SVT, the procedure halts after a fixed update budget, and can therefore handle many queries accurately as long as threshold crossings are rare. 
We extend this idea by allowing more general stopping rules, so the number of updates may be data-dependent, which better matches our setting. 
Rather than enforcing a hard update cap, we only require the stopping rule to control the number of updates with high probability. We state the privacy guarantee of Algorithm~\ref{alg:adjusted_sparse} in Theorem~\ref{thm:DP_adjusted} and a detailed proof can be found in Appendix~\ref{sec:proof}.
\begin{theorem}\label{thm:DP_adjusted}
Given a private database \(D\), an adaptively chosen stream of \(\Delta\)-sensitive queries \(f_1,f_2,\ldots\), a stopping rule and a threshold \(\lambda\), let \(B\) denote the number of updates. Set $\sigma \;=\; \frac{\Delta\sqrt{32\ln(1/\delta)}}{\epsilon}$.
If \(\Pr[B>k]\le \delta'\), then Algorithm~\ref{alg:adjusted_sparse} is \((\epsilon\sqrt{k},\,\delta+\delta')\)-DP.
\end{theorem}
\begin{algorithm}[ht]
\caption{Generalized Sparse}
\label{alg:adjusted_sparse}
\begin{algorithmic}
\REQUIRE{A private database $D$, an adaptively chosen stream of $\Delta$-sensitive queries $f_1,f_2,\ldots$, a threshold $\lambda$, differential privacy parameters $(\epsilon,\delta)$, and a stopping rule; output alphabet $\{\top,\bot\}$, where $\top$ denotes ``threshold crossed'' and $\bot$ otherwise.}
\STATE{
    $\sigma \gets \frac{\Delta\sqrt{32\ln \frac{1}{\delta}}}{\epsilon},$\;$\text{count} \gets 0$.\;}
\STATE{Sample $\xi_{\text{count}}\sim\mathrm{Lap}(\sigma),$\; $\widehat{\lambda}_0 \gets \lambda + \xi_{\text{count}}$.\;}
\FOR{each query $i$}
\STATE{Sample $\mu_i\sim\text{Lap}(2\sigma)$}
\IF{$f_i(D)+\mu_i\geq\widehat{\lambda}_{\text{count}}$}
\STATE{Output $a_i=\top$,\;$\text{count}=\text{count}+1$.}
\STATE{Sample $\xi_{\text{count}}\sim\mathrm{Lap}(\sigma),$\; $\widehat{\lambda}_{\text{count}} \gets \lambda + \xi_{\text{count}}$.\;}
\ELSE\STATE{Output $a_i=\bot$.}
\ENDIF
\IF{the stopping rule is satisfied}
\STATE{Halt.}
\ENDIF
\ENDFOR
\end{algorithmic}
\end{algorithm}

\begin{remark}
When \(\Delta=1\) and the stopping rule is set as \(\text{count}> c\), we have \(k=c\) and \(\delta'=0\). 
In this case, Theorem~\ref{thm:DP_adjusted} reduces to the classical SVT guarantee of~\cite{DworkNRRV09}.
\end{remark}

\section{\ShakyPrepend: Algorithm and Guarantees}
\label{sec:Shaky Prepend}
\par We present the \ShakyPrepend algorithm in this section. At a high level, the algorithm repeatedly ``fixes'' the group on which the current predictor performs worst, by comparing its loss gap to the best achievable predictor for that group. As shown in Algorithm~\ref{alg:shakyprepend}, each update appends a new $(g_t,h_t)$ pair and yields a predictor represented as a decision list via the recursive operator
$
[g_t,h_t,\ldots,g_0,h_0](x)
:= g_t(x)\,h_t(x) + \bigl(1-g_t(x)\bigr)\,[g_{t-1},h_{t-1},\ldots,g_0,h_0](x),
$
with base case $[g_0,h_0](x):=g_0(x)h_0(x)$.
At each iteration, we replace the current predictor by selecting the best $h_t\in\mathcal{H}$ for the chosen group $g_t$. Equivalently, the list is evaluated from front to back: earlier $(g,h)$ pairs take priority in determining the final prediction. This matches the decision-list form used in~\citet{tosh2024simplenearoptimalalgorithmshidden}.

\par Because such adaptive comparisons can lead to overfitting, we inject carefully scaled noise into the update rule so that the resulting interaction can be cast as an instance of the Sparse Vector Technique (SVT) applied to adaptively chosen threshold queries. At a high level, this noise makes the procedure \emph{stable}---limiting how much any single data point can influence the sequence of adaptive decisions---which yields sharper \emph{generalization} and, ultimately, a better rate.

\begin{remark}
If we remove the Laplace noise injection, the resulting procedure can be viewed as a variant of the \Prepend algorithm of~\citet{tosh2024simplenearoptimalalgorithmshidden}, with a modified stopping rule that explicitly accounts for the group size. We refer to this variant as \GroupPrepend and give generalization guarantees for it in Appendix~\ref{sec:proof}.
\end{remark}
\begin{algorithm}[ht]
\caption{\ShakyPrepend}
\label{alg:shakyprepend}
\begin{algorithmic}
\REQUIRE Groups $\mathcal{G}$, hypothesis class $\mathcal{H}$, i.i.d.\ examples $(x_1, y_1), \ldots, (x_n, y_n)$ from $\mathcal{D}$, error bound $\lambda$, and noise scale $\sigma$.
\STATE Compute $h_0 \in \arg\min_{h \in \mathcal{H}} L_n(h)$
\STATE Set $f_0 = [1,h_0]$.
\STATE Sample $\xi_0\sim\text{Lap}(\sigma)$ and set $\lambda_0=\lambda + \xi_0$.
\FOR{$t = 0, 1, \ldots$}
    \STATE update $\gets$ False
    \FOR{$(g,h)\in\mathcal{G}\times\mathcal{H}$}
        \STATE Sample $\mu_{t,g,h}\sim \text{Lap}(2\sigma)$.
        \IF{$P_n(g)\bigl(L_n(f_t \mid g) - L_n(h \mid g)\bigr)+\mu_{t,g,h} \geq \lambda_t$}
            \STATE $(g_{t+1},h_{t+1}) \gets (g,h)$;\quad $\xi_t' \gets \mu_{t,g,h}$
            \STATE $f_{t+1} \gets [g_{t+1}, h_{t+1}, g_t, h_t, \ldots, g_0, h_0]$
            \STATE update $\gets$ True
            \STATE \textbf{break}
        \ENDIF
    \ENDFOR
    \IF{update}
        \STATE Sample $\xi_{t+1}\sim\text{Lap}(\sigma)$ and set $\lambda_{t+1}=\lambda+\xi_{t+1}$
    \ELSE
        \STATE Return $f_t$
    \ENDIF
\ENDFOR
\end{algorithmic}
\end{algorithm}
\par\noindent\textbf{Interactive (analyst--mechanism) view.}
We can formalize the procedure as a two-party interactive protocol between an \emph{adaptive analyst} (who selects queries based on the transcript so far) and a randomized mechanism (which returns a single bit \textsf{yes}/\textsf{no}).

\begin{itemize}
    \item \textbf{Initialization.}
    Given threshold $\lambda$, sample $\xi_{0}\sim \mathrm{Lap}(\sigma)$, set $\lambda_{0}=\lambda+\xi_{0}$, and initialize $\mathrm{count}=0$.

    \item \textbf{Analyst's move (round $i$).}
    Given the previous responses and the current state of the update rule, the analyst specifies a query
    $q_i:\mathcal{X}^n\to\mathbb{R}$ of the form
    \[
        q_i(S_n)\;=\;P_{S_n}(g_i)\Bigl[L_{S_n}(f_i\mid g_i)-L_{S_n}(h_i\mid g_i)\Bigr],
    \]
    where $(g_i,h_i,f_i)$ are chosen according to the following logic:

        If the last response was \textsf{yes}, set $f_i=[g_{i-1},h_{i-1},f_{i-1}]$ and enumerate $(g_i,h_i)\in\mathcal{G}\times\mathcal{H}$ from the beginning.
        
        If the last response was \textsf{no}, keep $f_i=f_{i-1}$ and continue enumerating $(g_i,h_i)\in\mathcal{G}\times\mathcal{H}$.
        
        If the last response was \textsf{no} and the enumeration of $\mathcal{G}\times\mathcal{H}$ is exhausted, output $f_i$ and terminate the protocol.

    \item \textbf{Mechanism's move (response to $q_i$).}
    Sample $\mu_i\sim \mathrm{Lap}(2\sigma)$. 
    \par If \(q_i(S_n)+\mu_i
        \ge\;\lambda_{\mathrm{count}}\), output \textsf{yes}, 
    set $\mathrm{count}\leftarrow \mathrm{count}+1$, sample
    $\xi_{\mathrm{count}}\sim \mathrm{Lap}(\sigma)$, and update
    $\lambda_{\mathrm{count}}=\lambda+\xi_{\mathrm{count}}$.
    
    Otherwise output \textsf{no}.
\end{itemize}
Our main result for Algorithm~\ref{alg:shakyprepend} is stated in Theorem~\ref{thm:shakyprepend}.
\begin{theorem}\label{thm:shakyprepend}
Suppose that $\mathcal{H}$ and $\mathcal{G}$ are finite, and define
$L_g^* \coloneqq \min_{h\in\mathcal{H}} L(h \mid g)$.
With probability at least $1-\beta$,
for all $g \in \mathcal{G}$,
\[
L(f\mid g)-L_g^*
\lesssim
\frac{n^{-\frac{2}{5}}\ln n}{P_n(g)}\,
\ln\!\left(\frac{24|\mathcal{G}||\mathcal{H}|}{\beta}\right)^{\frac{2}{5}}
\ln\!\left(\frac{2n|\mathcal{G}||\mathcal{H}|}{\beta}\right)^{\frac{1}{5}}.
\]
\end{theorem}
We now provide a proof sketch for intuition and a complete proof is deferred to Appendix~\ref{sec:proof}.
\begin{proof}[Proof sketch]
\noindent\textbf{Differential-Privacy Property of the Algorithm:}
Let $B$ denote the total number of \emph{updates} made by the algorithm. We first show that $B$ is upper-bounded with high probability:
\begin{lemma}[Upper Bound for Update Times]\label{lemma:update}
Denote $\alpha=\min_h L_n(h)$. Then,
\[
\Pr\!\left[B\leq \frac{2\alpha}{\lambda}\right]\geq 1-e^{-\frac{\lambda}{4\sigma}}\frac{4\alpha}{\lambda}.
\]
\end{lemma}
This bound follows from the boundedness of the loss and the structure of the update criterion. Combining the result with Theorem~\ref{thm:DP_adjusted}, we prove that Algorithm~\ref{alg:shakyprepend} is differentially private:
\begin{lemma}[Differential Privacy]\label{lemma:DP_of_shaky}
Algorithm~\ref{alg:shakyprepend} is $\Big(\epsilon\sqrt{\frac{2\alpha}{\lambda}},\delta+e^{-\frac{\lambda}{4\sigma}}\frac{4\alpha}{\lambda}\Big)$-DP.
\end{lemma}

\noindent \textbf{Bounding the target gap:}
    For any group $g$, we decompose the excess risk as
 \begin{small}
      \[
    L(f\mid g)\!-\!\min_{h}L(h\mid g)\!
    =\!
    \Bigl(L_n(f\mid g)-\min_{h}L_n(h\mid g)\Bigr)
    \!+\!
    \Bigl(L(f\mid g)-L_n(f\mid g)\Bigr)
    \!+\!
    \Bigl(\min_{h}L_n(h\mid g)-\min_{h}L(h\mid g)\Bigr).\]
 \end{small}

    The last two terms are controlled by a tight generalization bound for any differentially-private mechanism:
\begin{lemma}[Generalization Bound]\label{lemma:gen_bound}
Suppose $f_B$ is the output of the Algorithm~\ref{alg:shakyprepend} and $n\geq\frac{1}{\epsilon^2\frac{2\alpha}{\lambda}}\ln\Big(\frac{4\epsilon\sqrt{\frac{2\alpha}{\lambda}}}{\delta+e^{-\frac{\lambda}{4\sigma}}\frac{4\alpha}{\lambda}}\Big)$.
\[
    \Pr\!\left[\max_{g\in\mathcal{G}}
\left|L_n(f_B\mid g)-L(f_B\mid g)\right|>\frac{36\epsilon\sqrt{\frac{2\alpha}{\lambda}}}{P_n(g)}\right] <\frac{2|\mathcal{G}|(\delta+e^{-\frac{\lambda}{4\sigma}}\frac{4\alpha}{\lambda})}{\epsilon}.
\]
\end{lemma}
The empirical gap term is controlled by the stopping rule together with tail bounds for the injected Laplace noise. Combining this with Lemma~\ref{lemma:gen_bound}, we obtain an error bound conditioned on the number of updates:
\begin{lemma}[Error Bound]\label{lemma:err_bound}
Denote $a=\max_{g,h}-\mu_{B,g,h}$, $L_g^*=\min_{h\in\mathcal{H}}L(h\!\mid g)$. When $n\geq\frac{\lambda}{2\epsilon^2\alpha}\ln(\frac{4\epsilon\sqrt{\frac{2\alpha}{\lambda}}}{\delta})$,
\begin{small}
    \[\Pr\!\left[\max_{g\in\mathcal{G}}(L(f_B\!\mid g)-L_g^*)
>\frac{36\epsilon\sqrt{\frac{2\alpha}{\lambda}}+\lambda+a+\xi_B+\epsilon\sqrt{\frac{2}{\lambda}}}{P_n(g)}\right] \leq \mathcal{O}\!\left(\frac{|\mathcal{G}|(\delta+e^{-\frac{\lambda}{4\sigma}}\frac{4\alpha}{\lambda})}{\epsilon}\right)
+ 8|\mathcal{H}|^2|\mathcal{G}|^2 e^{-\frac{2\epsilon^2 n}{81\lambda}}.
\]
\end{small}

\end{lemma}
With an appropriate choice of \(\lambda\), \(\epsilon\) and $\delta$, we obtain the final theorem.
\end{proof}

\subsection{Fractional Variants of Prepend-Like Algorithms}
When the label space is numerical (e.g., $\mathbb{R}^k$), the update rule of \ShakyPrepend can be written as
\[
f_{t+1}(x)
= f_t(x) + g_{t+1}(x)\bigl(h_{t+1}(x) - f_t(x)\bigr),
\]
which naturally suggests a fractional variant obtained by introducing a step-size parameter:
\[
f_{t+1}(x)
= f_t(x) + \eta\, g_{t+1}(x)\bigl(h_{t+1}(x) - f_t(x)\bigr),
\]
where $\eta\in(0,1]$ is the step size hyperparameter. The previous stopping rule can also be viewed from a complementary perspective: since $L_n(h\mid g)=L_n([g,h,f_t]\mid g)$, the algorithm stops precisely when updating along any pair $(g,h)$ fails to improve the empirical conditional loss on the corresponding group $g$. Thus, the stopping rule can be easily extended in the fractional version by replacing $L_n(h\mid g)$ with $L_n(f'\mid g)$, where $f'=f_t + \eta\, g_{t+1}\bigl(h_{t+1} - f_t\bigr)$.

When $\eta<1$, each update moves only partially toward the group-specific best response $h_{t+1}$. In this regime, it is natural to continue updating whenever the update promises a substantial reduction in the group-conditional loss, rather than updating only when the current predictor is substantially worse than the best response. Overall, the $\eta$-update performs a fractional interpolation between the current predictor and the target update. The \FractionalShakyPrepend (the full algorithm is deferred to Appendix~\ref{sec:fraction}) has a similar theoretical guarantee as the original \ShakyPrepend, as shown in Theorem~\ref{thm:fraction}. Even without improving the theory, step-size tuning can substantially boost practical performance (Section~\ref{sec:Experiment}), since smaller $\eta$ yields a richer family of intermediate predictors.

\begin{theorem}\label{thm:fraction}
Suppose that $\mathcal{H}$ and $\mathcal{G}$ are finite, the loss function $\ell$ is convex with respect to $f(x)$, and define
$L_g^* \coloneqq \min_{h\in\mathcal{H}} L(h \mid g)$. With probability $1-\mathcal{O}(\beta)$, Algorithm~\ref{alg:smoothedshakyprepend} terminates with predictor $f$ satisfying, for all $g\in\mathcal{G}$,
\begin{small}
    \[
L(f\mid g)-L_g^*
\lesssim
\frac{n^{-\frac{2}{5}}\ln n}{\eta P_n(g)}\,
\ln\!\left(\frac{4|\mathcal{G}||\mathcal{H}|}{\beta}\right)^{\frac{2}{5}}
\ln\!\left(\frac{2n|\mathcal{G}||\mathcal{H}|}{\beta}\right)^{\frac{1}{5}}.
\]
\end{small}
\end{theorem}
\begin{remark}
Beyond \FractionalShakyPrepend, analogous variants of \Prepend\ and \GroupPrepend\ can be proposed and admit the same guarantees by essentially the same proof.
\end{remark}
\subsection{Relation with Gradient Boosting}
Let $A(f)$ be an objective functional of a predictor $f$. Gradient boosting aims to solve
$\inf_{f\in \operatorname{span}(S)} A(f)$,
where $S$ is a collection of real-valued functions and
\[
\operatorname{span}(S)
=
\left\{
\sum_{j=1}^{m} w^{j} f^{j}
:\; f^{j}\in S,\; w^{j}\in \mathbb{R},\; m\in \mathbb{Z}^{+}
\right\}.
\]
Starting from an initial predictor $f$, at each iteration the algorithm approximately solves
\(
\arg\min_{a\in\mathbb{R},\, g\in S} A(f + a g),
\)
and then updates $f \leftarrow f + a g$.

Algorithm~\ref{alg:shakyprepend} and Algorithm~\ref{alg:smoothedshakyprepend} can both be interpreted as a fixed-step-size variant of gradient boosting, equipped with an early-stopping rule that acts as a regularizer, and with update directions that depend on the current predictor. Concretely, at iteration $t$ we choose an update direction from the set
\[
\left\{\, \eta\,\mathbbm{1}_{\{g(x)=1\}}\,(h_g - f_t)\ :\ g\in \mathcal{G} \,\right\},
\]
where $\eta$ is the step size, $h_g=\arg\min_{h\in\mathcal{H}}L_n(h|g)$, $f_t$ denotes the current predictor. Instead of updating globally at each step, the algorithm updates only locally—restricted to a selected group—along the corresponding direction in this set, yielding strong group-wise guarantees.

\section{Experiments}
\label{sec:Experiment}
\par We present simulation experiments evaluating \ShakyPrepend and other multi-group learning algorithms from three angles: (i) \textbf{practical guidance}, focusing on parameter tuning and the choice of tuning criterion; (ii) \textbf{properties}, including group-size and spatial adaptivity;  and (iii) \textbf{variants}, assessing whether the fractional version improves performance in practice.
The code is publicly available at \href{https://github.com/lujingz/shaky_prepend}{https://github.com/lujingz/shaky\_prepend}.

\noindent \textbf{Baselines.}
We consider the multi-group learning methods of~\citet{tosh2024simplenearoptimalalgorithmshidden}: \Prepend, \GroupPrepend, and \SleepingExpert, as reviewed in Section~\ref{sec:Shaky Prepend}. Among them, \SleepingExpert is a randomized, online-learning-based algorithm that attains optimal rates, but it is less space- and time-efficient in practice due to storing intermediate policies and requiring sampling at inference.

\noindent \textbf{Plot conventions.}
Unless otherwise stated, points denote the mean over independent runs and error bars show $\pm$ one standard error; each run is evaluated on a fixed test set.

\subsection{Criterion Selection: Guidance for Hyperparameter Tuning}
\par In practice, since it is impossible to calculate the exact value of the hyperparameters, the standard solution is to use a validation set to fine-tune them. However, in the multi-group setting, there are many natural choices of the criterion when fine-tuning. We mainly discuss two criteria here: the global loss $L(f)$ and the worst group loss $\max_{g\in\mathcal{G}}L(f|g)$. Our experiments suggest the following intuitive guidelines: with enough validation data, tune for the application’s target metric; with limited data, worst-group tuning suffers high variance, so tuning by the global loss is often more reliable.
  \begin{figure}[ht]
      \centering
      \includegraphics[width=0.7\linewidth]{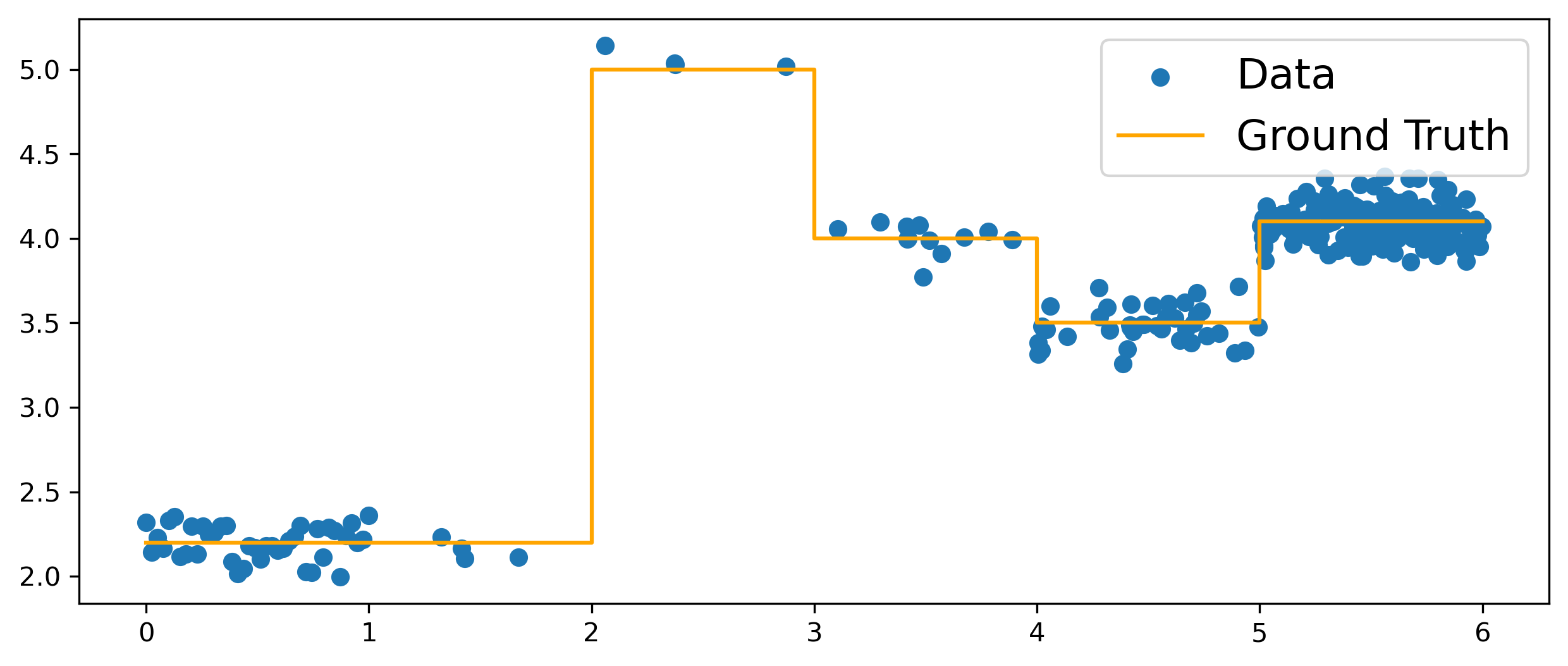}
      \includegraphics[width=0.65\linewidth]{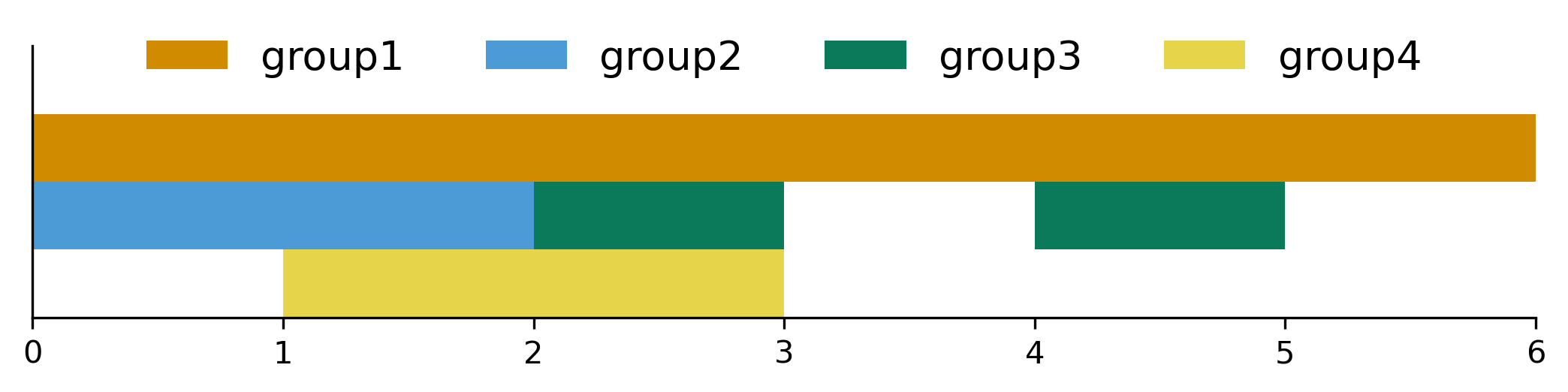}
      \caption{\textbf{Criterion-selection setup.} \emph{Top:} Ground truth (orange) and noisy training samples (blue). \emph{Bottom:} Groups are constructed so that the worst-group loss is not aligned with the total loss.}
      \label{fig:criterion_selection_setup}
  \end{figure}
The setup is illustrated in Figure~\ref{fig:criterion_selection_setup}. It is designed so that optimizing the total loss conflicts with optimizing the worst-group loss: one candidate solution attains a smaller total loss, while another achieves a smaller worst-group loss. A more detailed discussion is deferred to Appendix~\ref{sec:experiment details}. 
\begin{figure}[ht]
    \centering
    \includegraphics[width=0.48\linewidth]{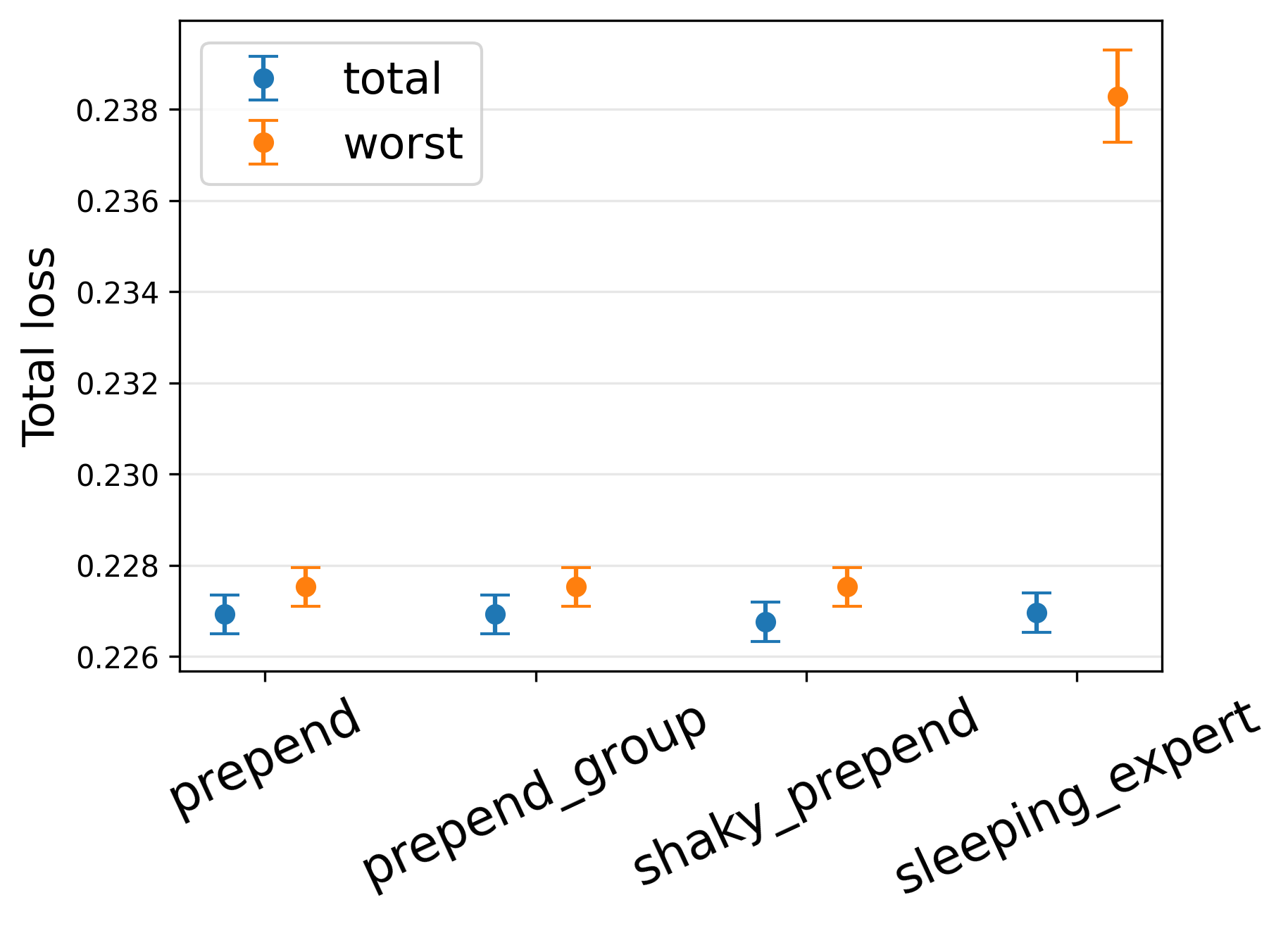}
    \includegraphics[width=0.48\linewidth]{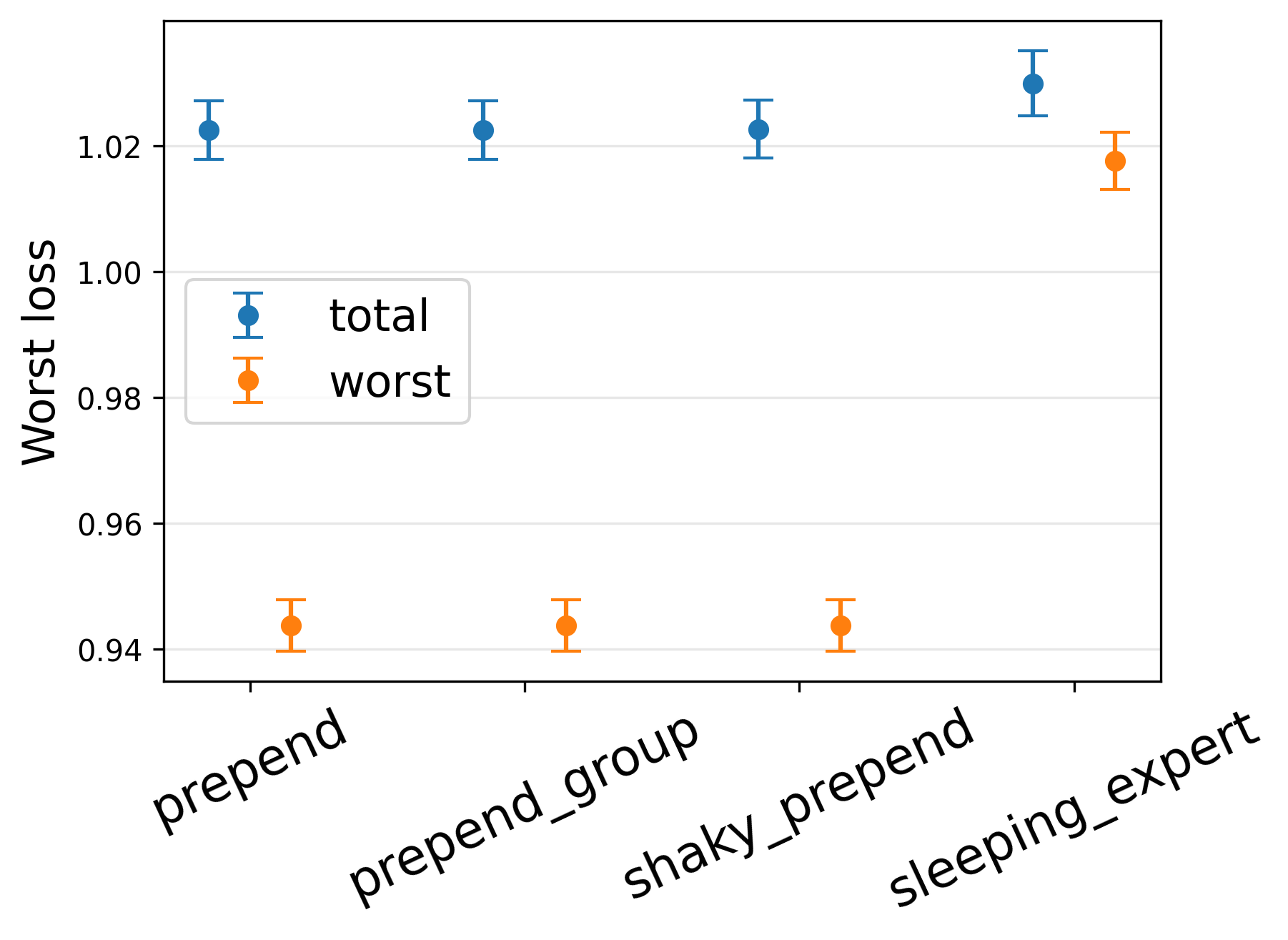}
\caption{\textbf{Criterion selection (large sample).}
Total loss (left) and worst-group loss (right) when tuning hyperparameters by total loss vs.\ worst-group loss (20 runs; 26{,}000 training points per run). }
    \label{fig:criterion_selection}
\end{figure}
\begin{figure}[ht]
    \centering
    \includegraphics[width=0.48\linewidth]{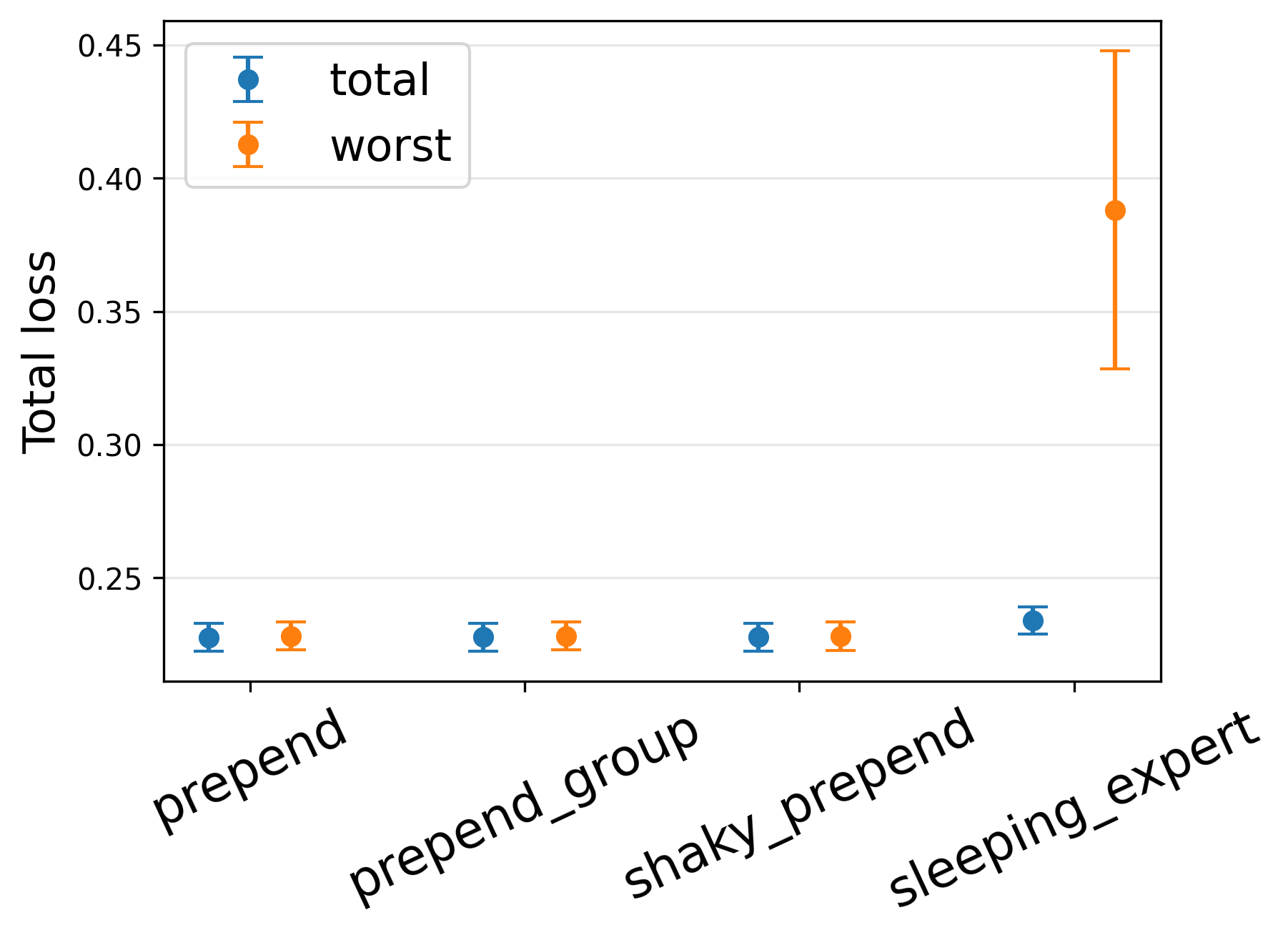}
    \includegraphics[width=0.48\linewidth]{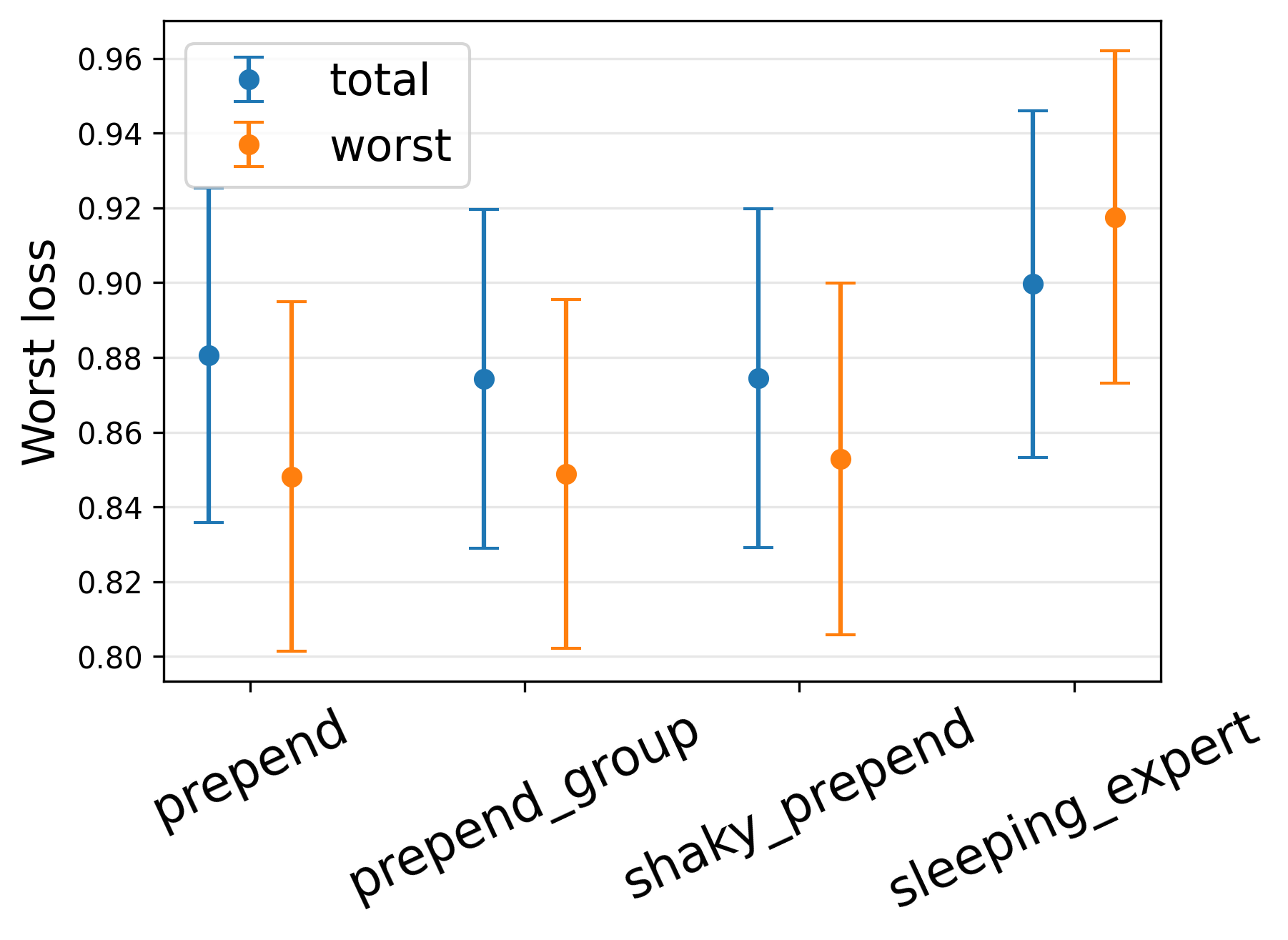}
    \caption{\textbf{Criterion selection (small sample).}
Total loss (left) and worst-group loss (right) when tuning hyperparameters by total loss vs.\ worst-group loss (20 runs; 260 training points per run).}
    \label{fig:criterion_selection2}
\end{figure}

Figure~\ref{fig:criterion_selection} shows that when the sample size is large, tuning hyperparameters for worst-group loss improves worst-group performance, while tuning for total loss yields better total-loss performance. Notably, worst-group tuning is less effective for \SleepingExpert: it substantially increases total loss while only slightly reducing worst-group loss.

When the sample size is small, Figure~\ref{fig:criterion_selection2} shows that worst-loss tuning can be unreliable, especially for \SleepingExpert. In particular, \SleepingExpert tuned by worst loss can even attain a \emph{larger} worst-group loss than when tuned by total loss. This is expected: worst-group loss is a high-variance criterion, so hyperparameter selection based on a noisy estimate can degrade performance. The effect is particularly pronounced for stochastic methods such as \SleepingExpert, whose inherent randomness further reduces stability.

\subsection{Unbalanced-Group Setting: Group-Size Adaptivity of \GroupPrepend and \ShakyPrepend}
\par In many practical settings, the data are group-wise unbalanced, raising the question of whether one should use a potentially biased predictor trained on the larger group or a high-variance predictor trained on the minority group.
\par In this simulation study, we demonstrate that \GroupPrepend and \ShakyPrepend can automatically balance these two extremes, achieving consistently better performance than the original \Prepend method.

\begin{figure}[ht]
    \centering
    \includegraphics[width=.6\linewidth]{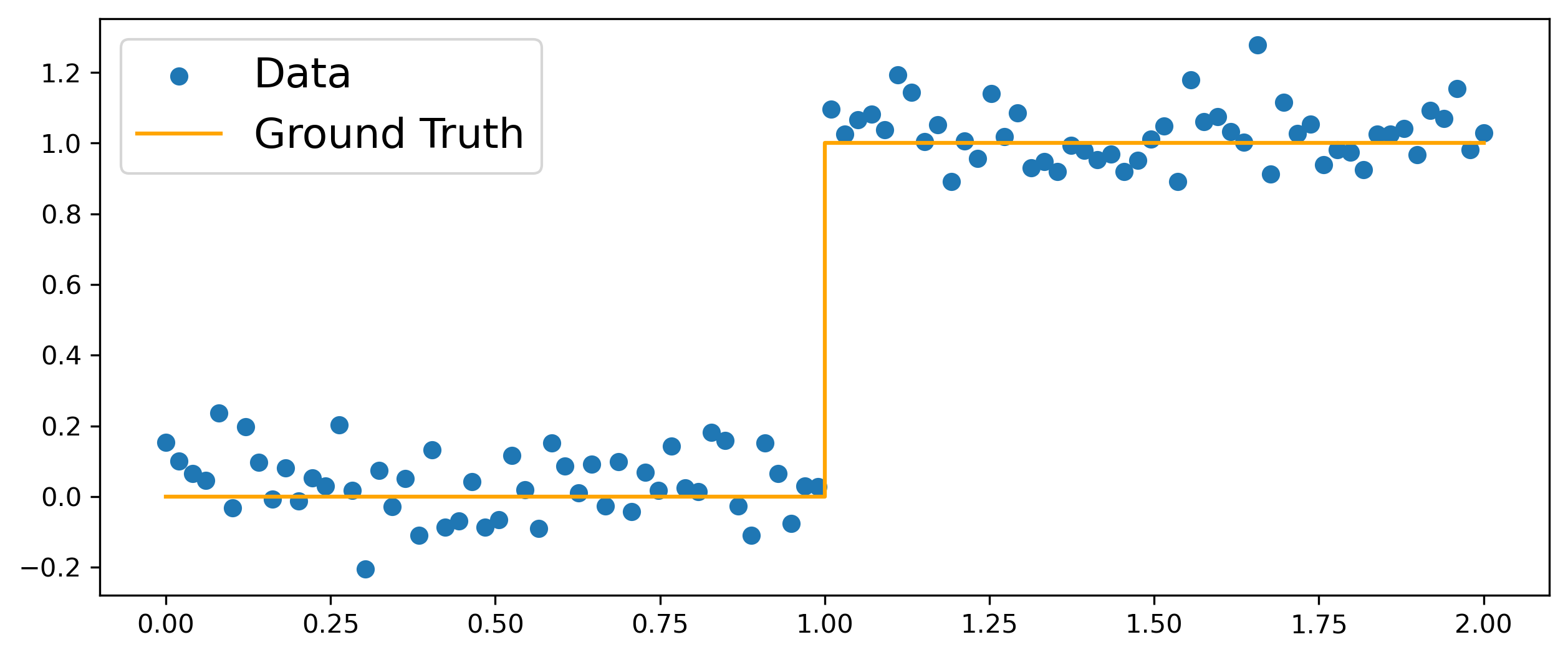}
    \includegraphics[width=.6\linewidth]{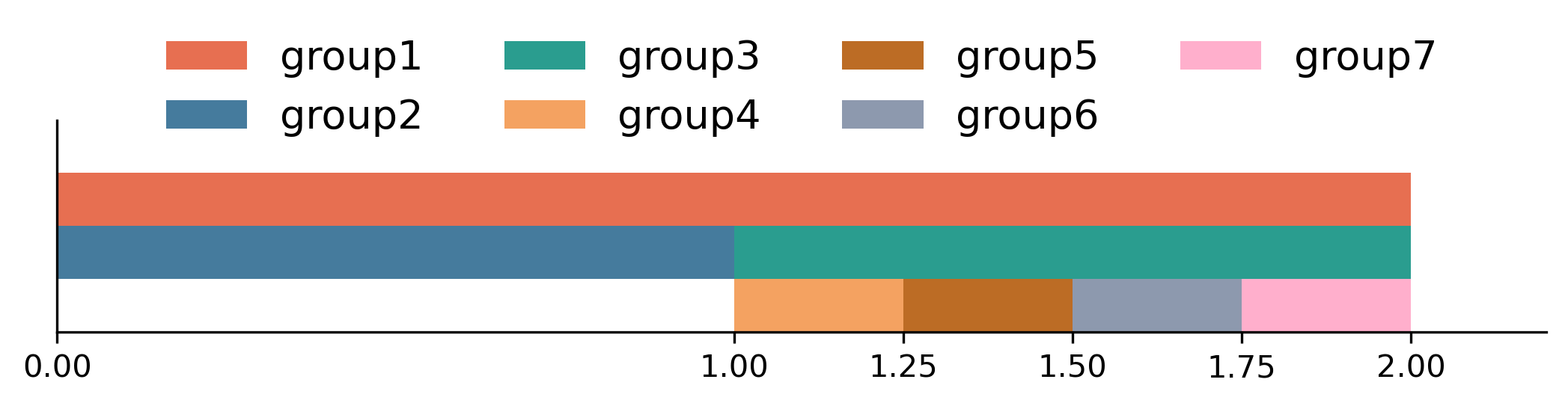}
    \caption{\textbf{Unbalanced-group setup.}
\emph{Top:} Ground truth (orange) and noisy training samples (blue).
\emph{Bottom:} Layered group intervals, where only part of the domain is further refined into smaller subgroups, yielding unbalanced granularity.}
    \label{fig:unbalanced_experiment_construction}
\end{figure}
\par As illustrated in Figure~\ref{fig:unbalanced_experiment_construction}, we use two coarse groups (groups 2 and 3) and four finer groups (groups 4--7) to mimic a multi-granularity group structure (e.g., clinical cohorts where some conditions admit meaningful subtype stratification while others do not).
\begin{figure}[ht]
    \centering
    \includegraphics[width=\linewidth]{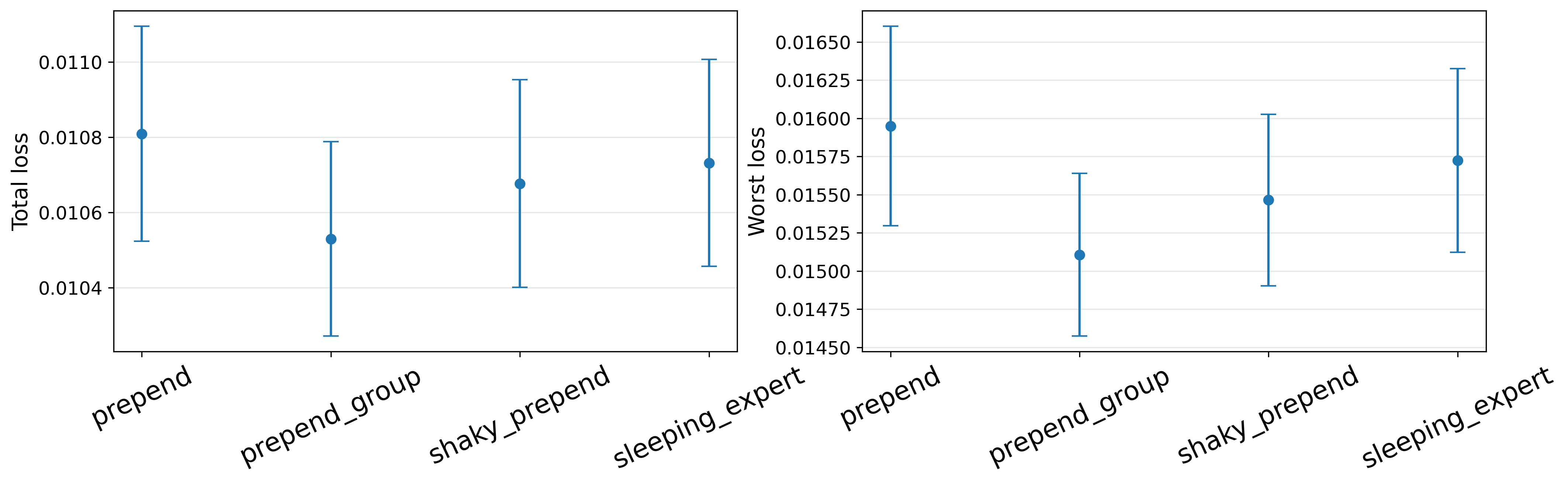}
\caption{\textbf{Unbalanced-group simulation.}
Total loss (left) and worst-group loss (right) for \Prepend, \GroupPrepend, and \ShakyPrepend (20 runs; 120 training points per run).}
    \label{fig:unbalanced_experiment}
\end{figure}
\par Figure~\ref{fig:unbalanced_experiment} illustrates that the \GroupPrepend achieves the best performance on both the total loss and the worst-case loss, followed by \ShakyPrepend, then \SleepingExpert and \Prepend. This ranking is expected: \GroupPrepend and \ShakyPrepend incorporate group size into their stopping rules, which implicitly balances bias and variance, especially for small groups.

\subsection{Spatial Adaptivity}

\par Many real-world targets exhibit \emph{spatial inhomogeneity}, yet the locations and scales of the underlying regions are often unknown in advance. The experiment below demonstrates that, by constructing a rich collection of nested candidate groups, the multi-group learning algorithms tested can adapt to latent spatial features and produce accurate predictions.
\par Figure~\ref{fig:spatial} visualizes the ground-truth signal together with noisy observations, highlighting an underlying piecewise structure. When this spatial structure is not known \emph{a priori}, we consider the candidate group family to be intervals of varying lengths centered at grid points:
\[
\mathcal{G}
=\bigl\{[c-\tfrac{l}{2},\,c+\tfrac{l}{2}] : c\in[0,1]_{0.05},l\in(0,1]_{0.05}\},
\]
where $[0,1]_{0.05}$ denotes the $0.05$-spaced grid on $[0,1]$, i.e.,
\[
[0,1]_{0.05} := \{0, 0.05, 0.10, \ldots, 1\},
\]
and similarly $(0,1]_{0.05} := \{0.05, 0.10, \ldots, 1\}$.
 The goal is to test whether the algorithms can \textbf{automatically} recover the latent spatial structure by selecting suitable centers and interval lengths (i.e., groups). We let $\mathcal{H}$ be the class of groupwise-constant predictors.
\begin{figure}[ht]
    \centering
    \includegraphics[width=0.48\linewidth]{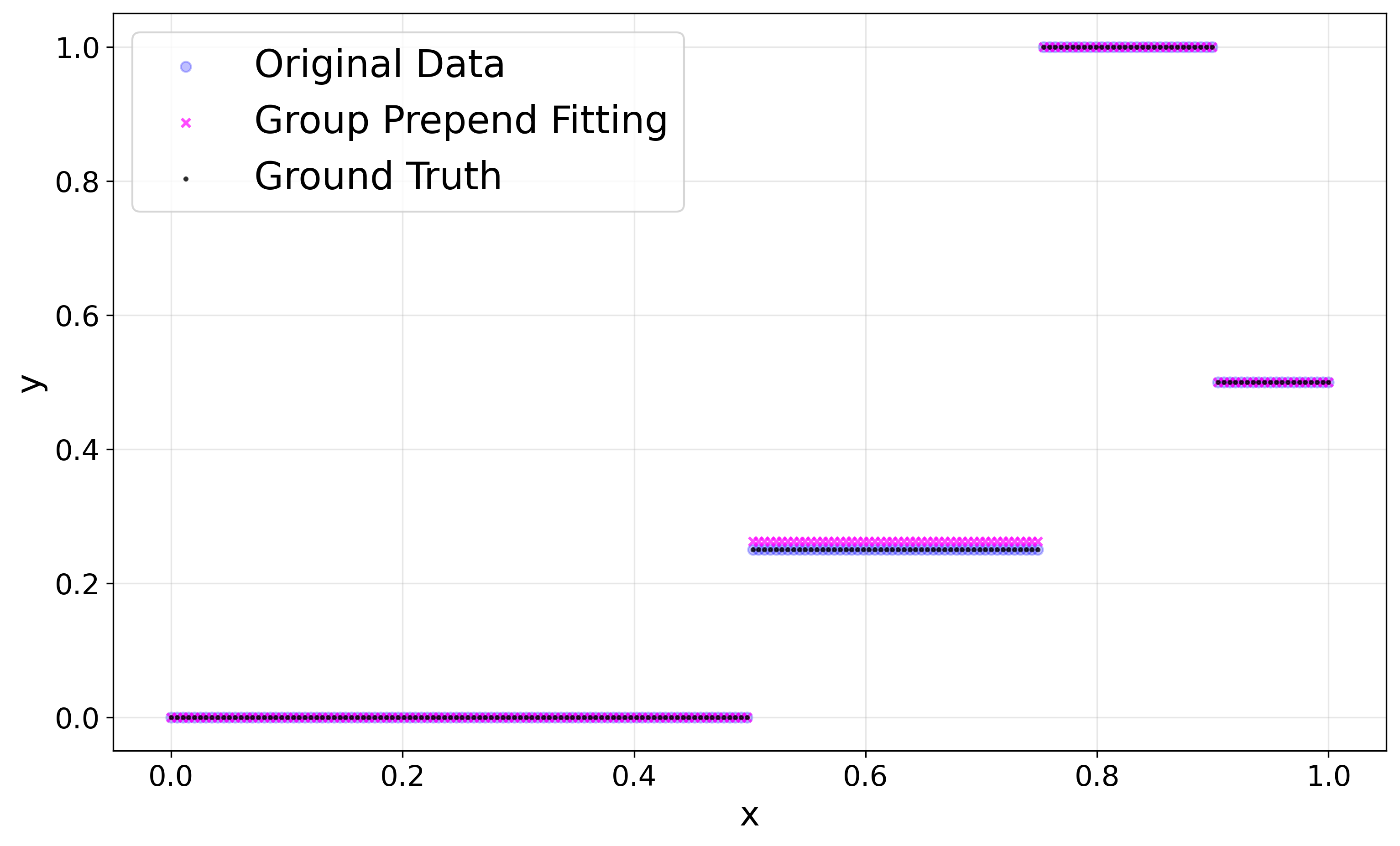}
    \includegraphics[width=0.48\linewidth]{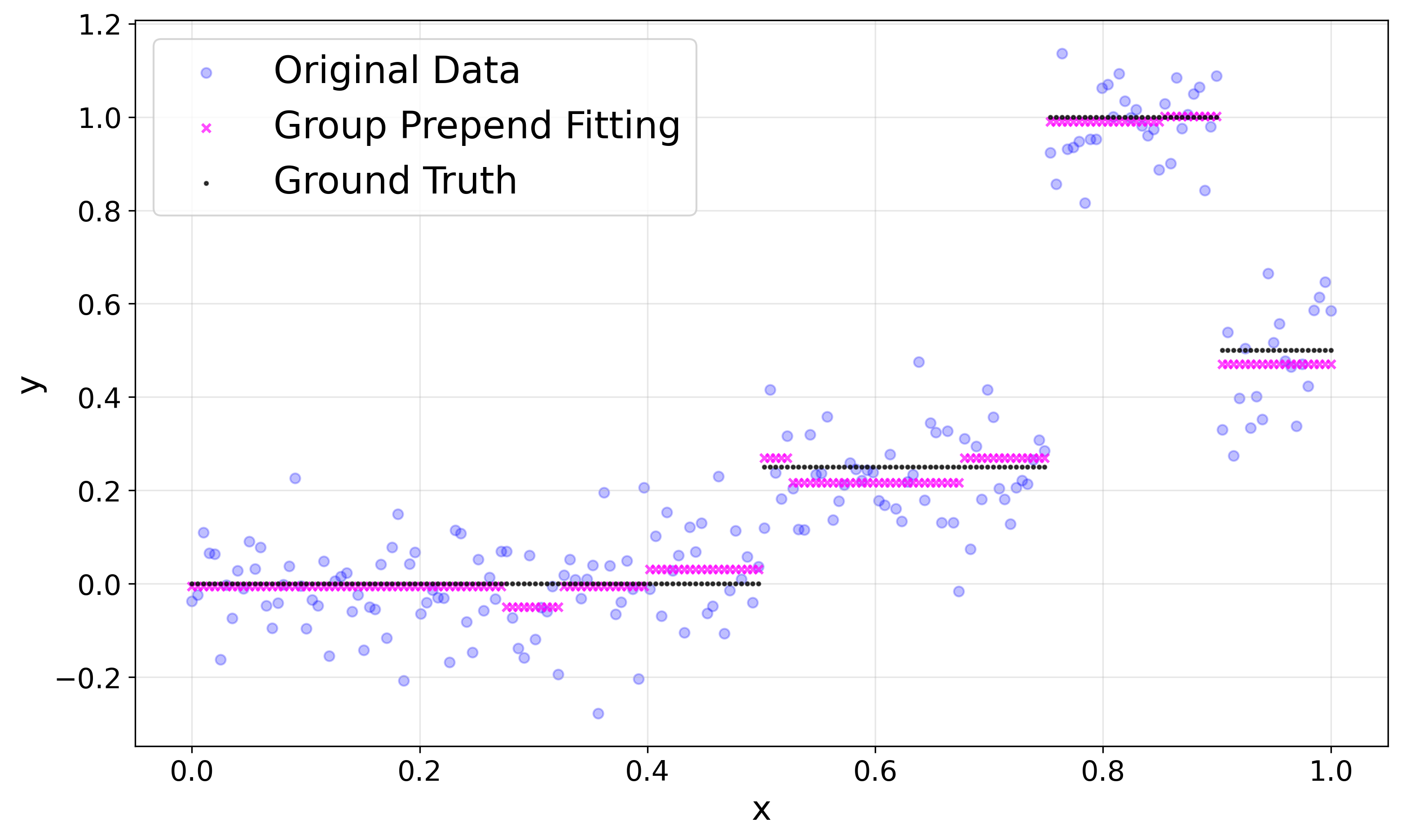}
    \caption{\textbf{Spatial adaptivity.}
    Ground truth, observations ($n=200$), and \GroupPrepend predictions under additive Normal noise with standard error in $\{0,0.1\}$ (rows).}
    \label{fig:spatial}
\end{figure}
\par Empirically, all methods adapt well to the target’s spatial structure and achieve reasonably accurate predictions. Due to space constraints, we report only \GroupPrepend in Figure\ref{fig:spatial} (with additional results deferred to Appendix~\ref{sec:experiment details}); across noise levels, it consistently captures the underlying piecewise spatial pattern, and while higher noise can occasionally lead to suboptimal choices of finer-grained groups, its overall performance remains strong.

\subsection{Fractional Variants of Prepend-Like Algorithms}
\par Although the fractional variant does not improve the theoretical bound, it can be viewed as exploring a richer function class by allowing fractional updates, which may yield practical gains. Our experiments support this intuition, especially for \ShakyPrepend. Specifically, we adopt the spatial-adaptivity setup and evaluate fractional variants of \Prepend, \GroupPrepend, and \ShakyPrepend. For each method, we compare the original algorithm to its fractional counterpart (full pseudocode is deferred to Appendix~\ref{sec:experiment details}), and we additionally tune the fractional step size over $\{0.5,1\}$. As shown in Figure~\ref{fig:errorplot_spatial}, the fractional variant consistently attains lower loss across methods, improving both total loss and worst-case group loss. These results suggest that fractional updates provide a simple and effective practical enhancement, even when they leave the worst-case theory unchanged.
\begin{figure}[ht]
    \centering
    \includegraphics[width=0.48\linewidth]{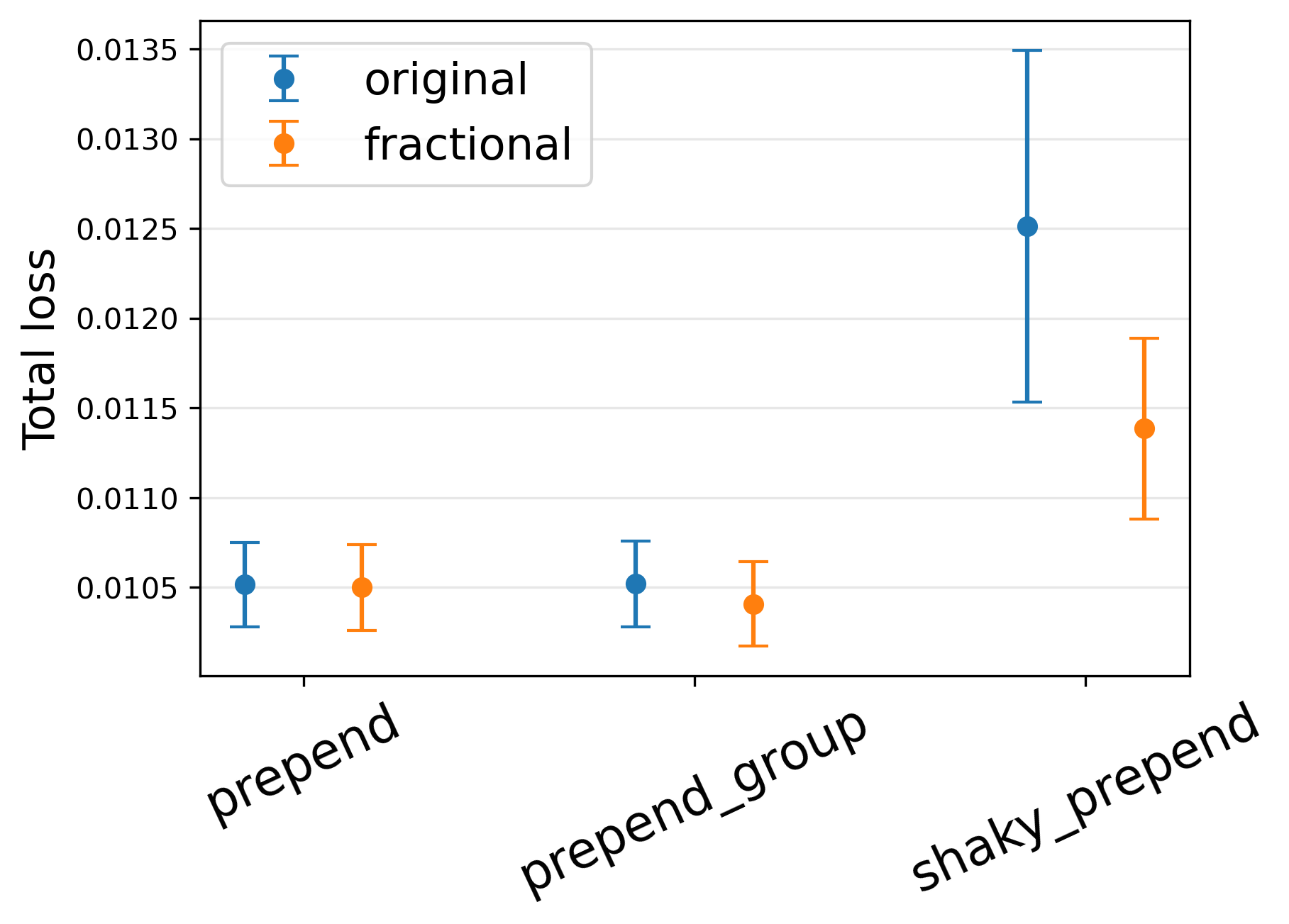}
    \includegraphics[width=0.48\linewidth]{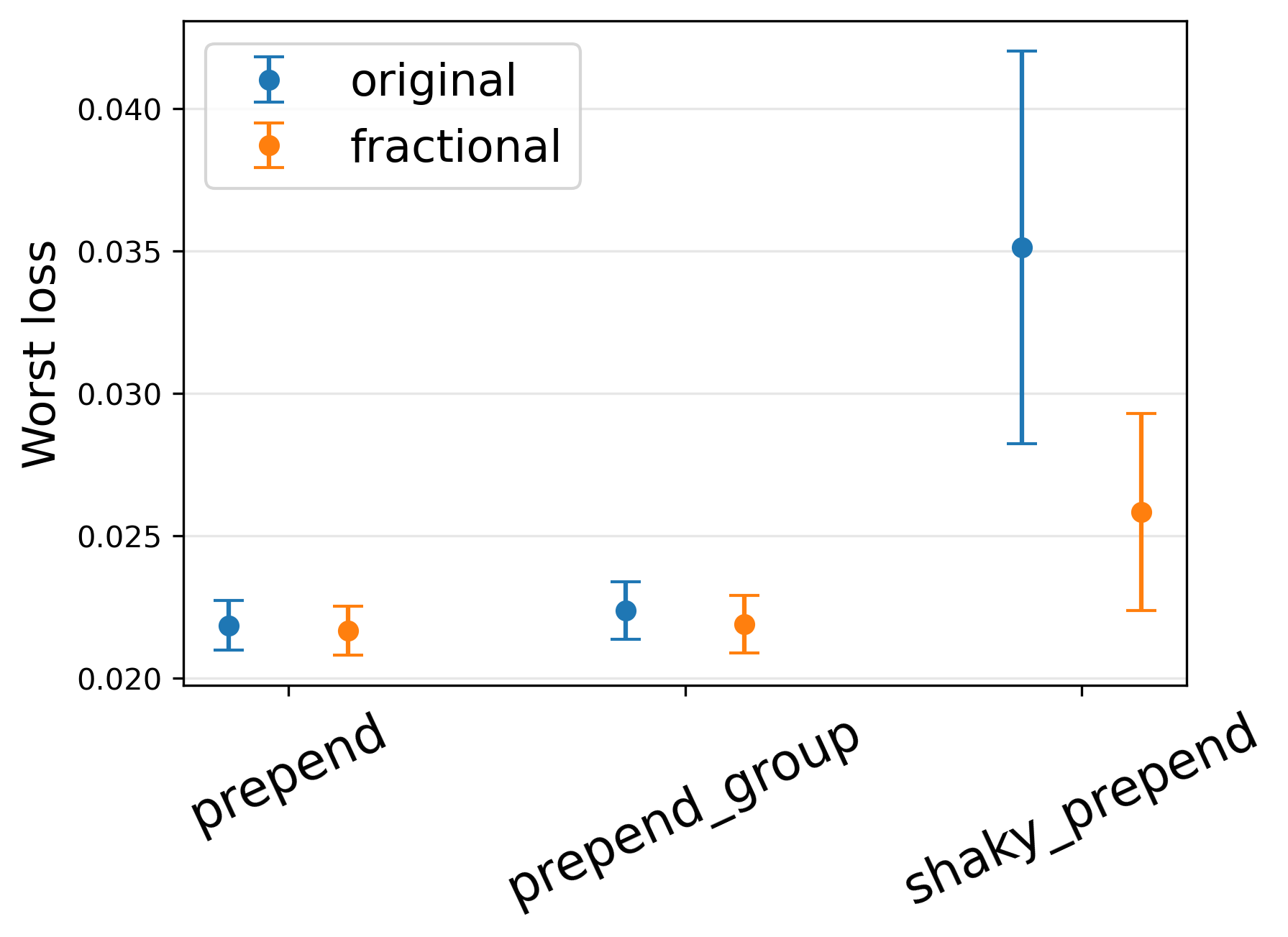}
    \caption{\textbf{Fractional ablation.}
    Total loss (left) and worst-group loss (right) comparing \ShakyPrepend and its fractional variant with tuned step size $\eta\in\{0.5,1\}$ (20 runs; 200 training points per run).}
    \label{fig:errorplot_spatial}
\end{figure}
\section{Conclusion}
\par We propose~\ShakyPrepend, a multi-group learning algorithm that leverages differential privacy to obtain improved sample complexity and group-size dependence. Several directions remain open:

\noindent \textbf{Extending to multicalibration and multiaccuracy.}
Multicalibration and multiaccuracy algorithms share a common iterative template: an auditor finds the most-violated subgroup (or constraint) and the predictor is updated accordingly. Several works (e.g.,~\citet{hebertjohnson2018multicalibration,gopalan2022lowdegreemulticalibration,haghtalab2023unifying}) have noted that differential privacy could reduce the sample complexity of such auditing-based procedures, but existing results are often high-level with loose bounds and limited empirical validation. It would be interesting to seek sharper guarantees via analyses similar to ours.

\par \noindent \textbf{Infinite (or non-enumerable) hypothesis/group classes.}
Our DP approach requires explicitly enumerating $\mathcal{H}$ and $\mathcal{G}$, which fails when either class is too large. It remains open whether DP techniques can be incorporated into oracle-efficient algorithms (e.g.,~\citet{deng2025groupwiseoracleefficientalgorithmsonline}) in this regime, while retaining computational efficiency and improving sample complexity.
\bibliography{reference}
\bibliographystyle{icml2026}
\newpage
\appendix
\onecolumn
\section{Experiment Details}\label{sec:experiment details}
\subsection{Hyperparameters} In \Prepend and \GroupPrepend, the error tolerance $\lambda$ is the only hyperparameter. \ShakyPrepend introduces an additional hyperparameter $\sigma$ that controls the magnitude of the injected noise. In the fractional variant, the step size $\eta$ is also treated as a hyperparameter. For \SleepingExpert, the learning rate is the hyperparameter.
\subsection{Explanation for the Construction of the Criterion Selection Case}
  \begin{figure}[ht]
      \centering
      \includegraphics[width=0.5\linewidth]{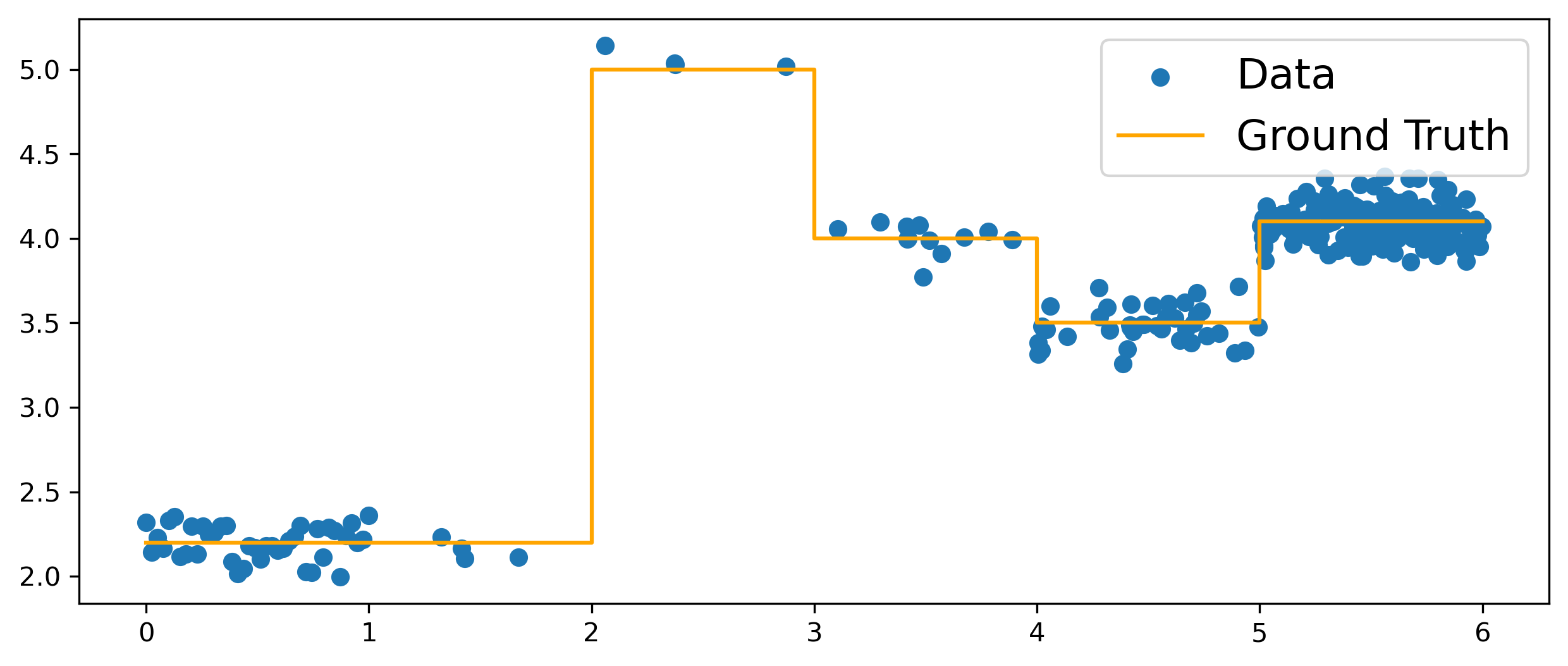}
      \includegraphics[width=0.45\linewidth]{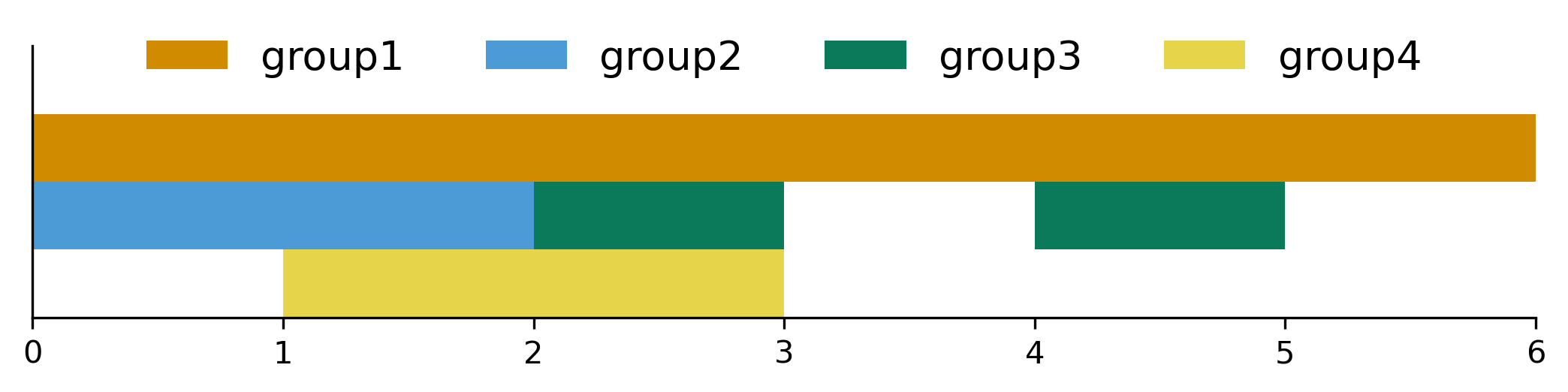}
      \caption{\textbf{Criterion-selection setup.} \emph{Left:} Ground truth (orange) and noisy training samples (blue). \emph{Right:} Groups are constructed so that the worst-group loss is not aligned with the total loss.}
  \end{figure}
 \par Group~4 is designed to be the likely worst group: it has the smallest sample size
 and its target values deviate the most from what a constant predictor would capture. Since most observations concentrate around $4.1$, the initial predictor is close to $4$. For Prepend-like algorithms, the first update will, with high probability, improve the fit on the interval $[0,2]$ while leaving the prediction on $[2,3]$ essentially unchanged. In contrast, updating Group~3 lowers the prediction on $[2,3]$, which can increase the loss on Group~4, even as it improves the fit on $[4,5]$ and thereby reduces the total loss.
As a result, hyperparameters tuned for total loss tend to encourage more aggressive updates (including updating Group~3), whereas hyperparameters tuned for worst-group loss prefer more conservative behavior and often avoid updating Group~3. This tension produces markedly different outcomes under the total-loss and worst-loss criteria, as illustrated in Figure~\ref{fig:criterion_selection}. For \SleepingExpert, the same trade-off persists—the best predictor on Group~3 benefits Group~3 but can hurt performance on its overlap with Group~4—though it manifests in a more complex, randomized way due to the algorithm’s stochastic expert weighting.
\newpage
\subsection{Additional Figures for the Spatial Adaptivity Experiment}
 \begin{figure}[ht]
     \centering
     \includegraphics[width=0.45\linewidth]{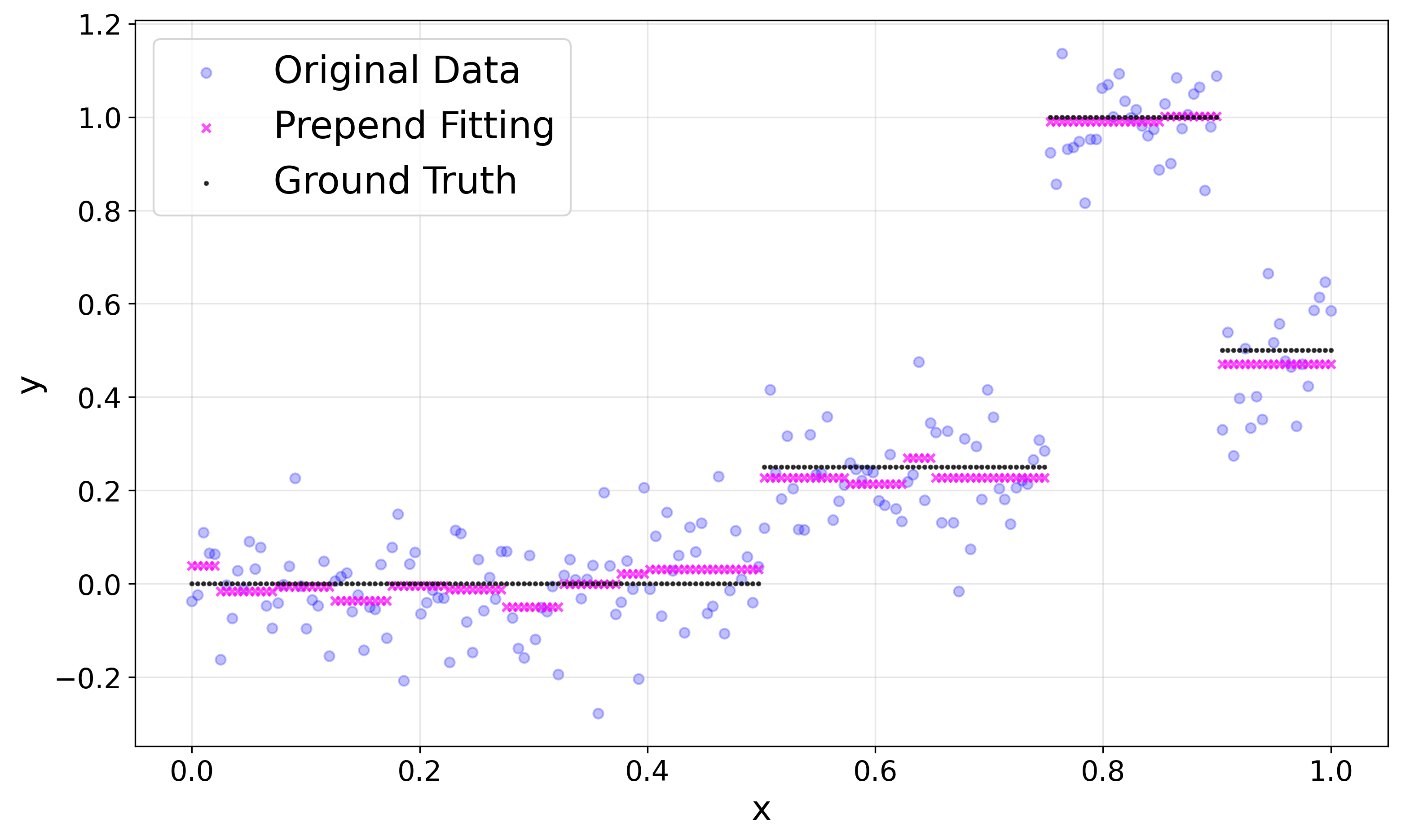}
      \includegraphics[width=0.45\linewidth]{figs/piecewise_experiment_group_prepend_only_0.1.png}
     \includegraphics[width=0.45\linewidth]{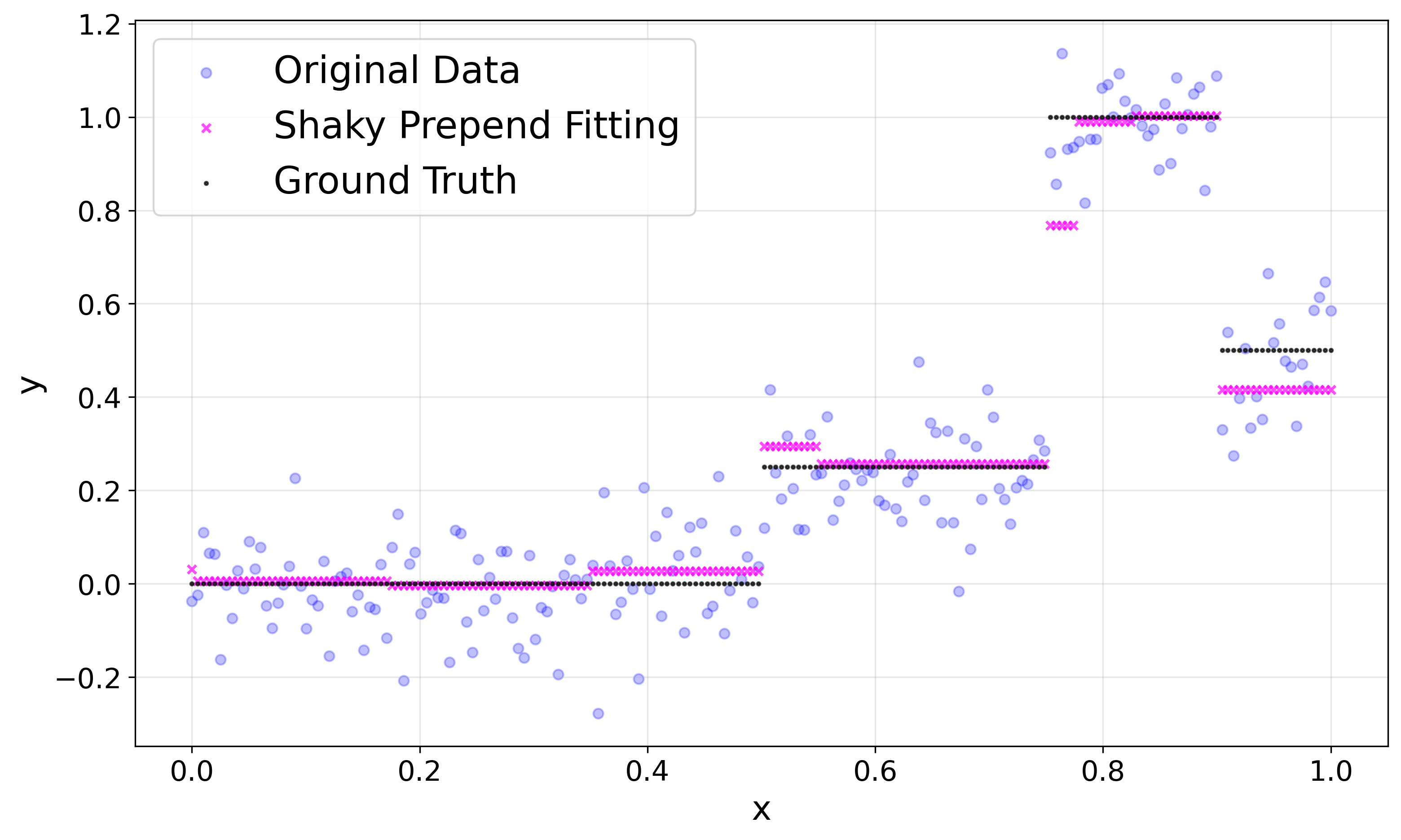}
  \includegraphics[width=0.45\linewidth]{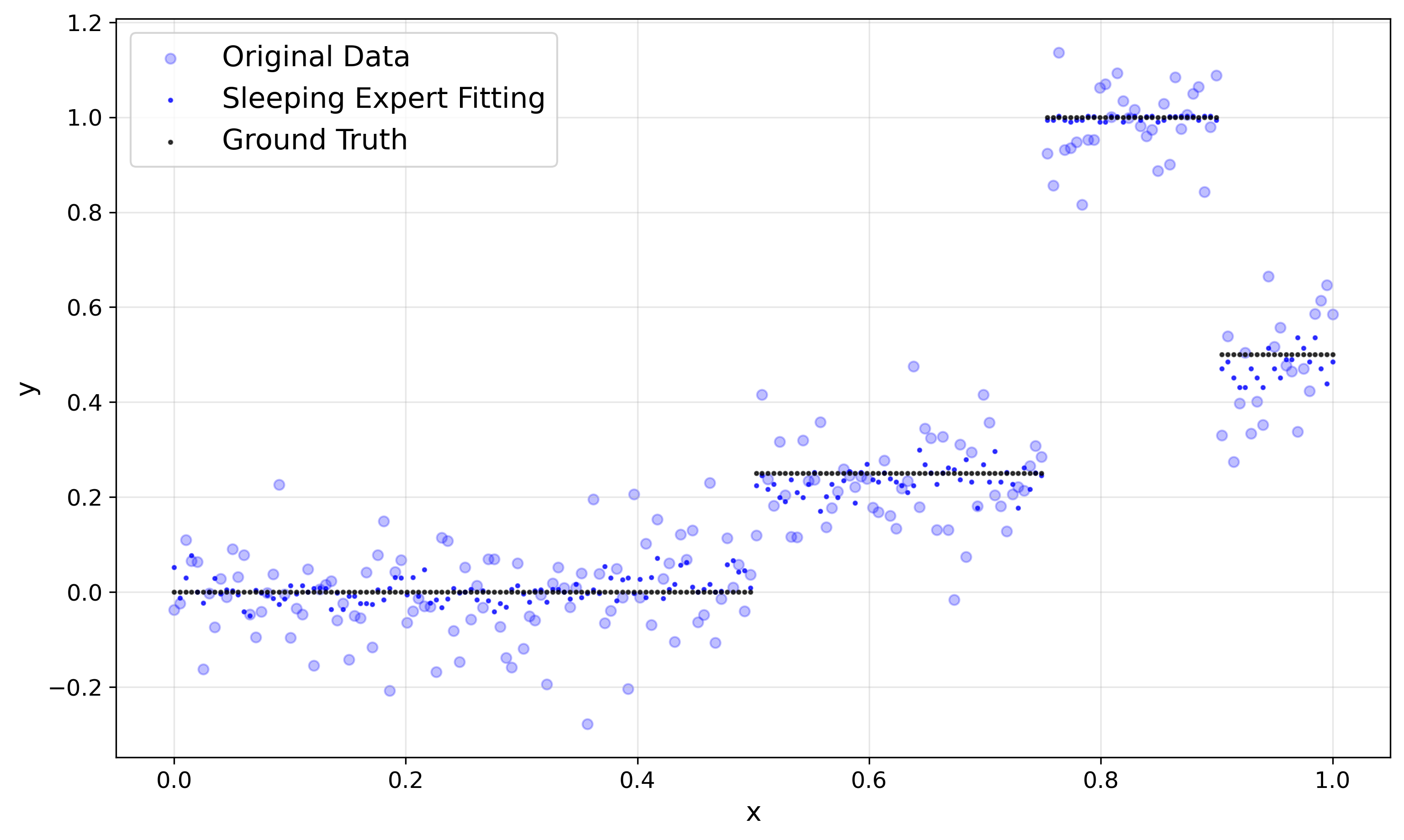}
     \caption{Predictions produced by \Prepend, \GroupPrepend, \ShakyPrepend, and \SleepingExpert (top-left to bottom-right) under additive noise with $\sigma=0.1$. All methods recover the target's spatial structure.}
     \label{fig:additional_plot}
 \end{figure}
  \begin{figure}[ht]
     \centering
     \includegraphics[width=0.45\linewidth]{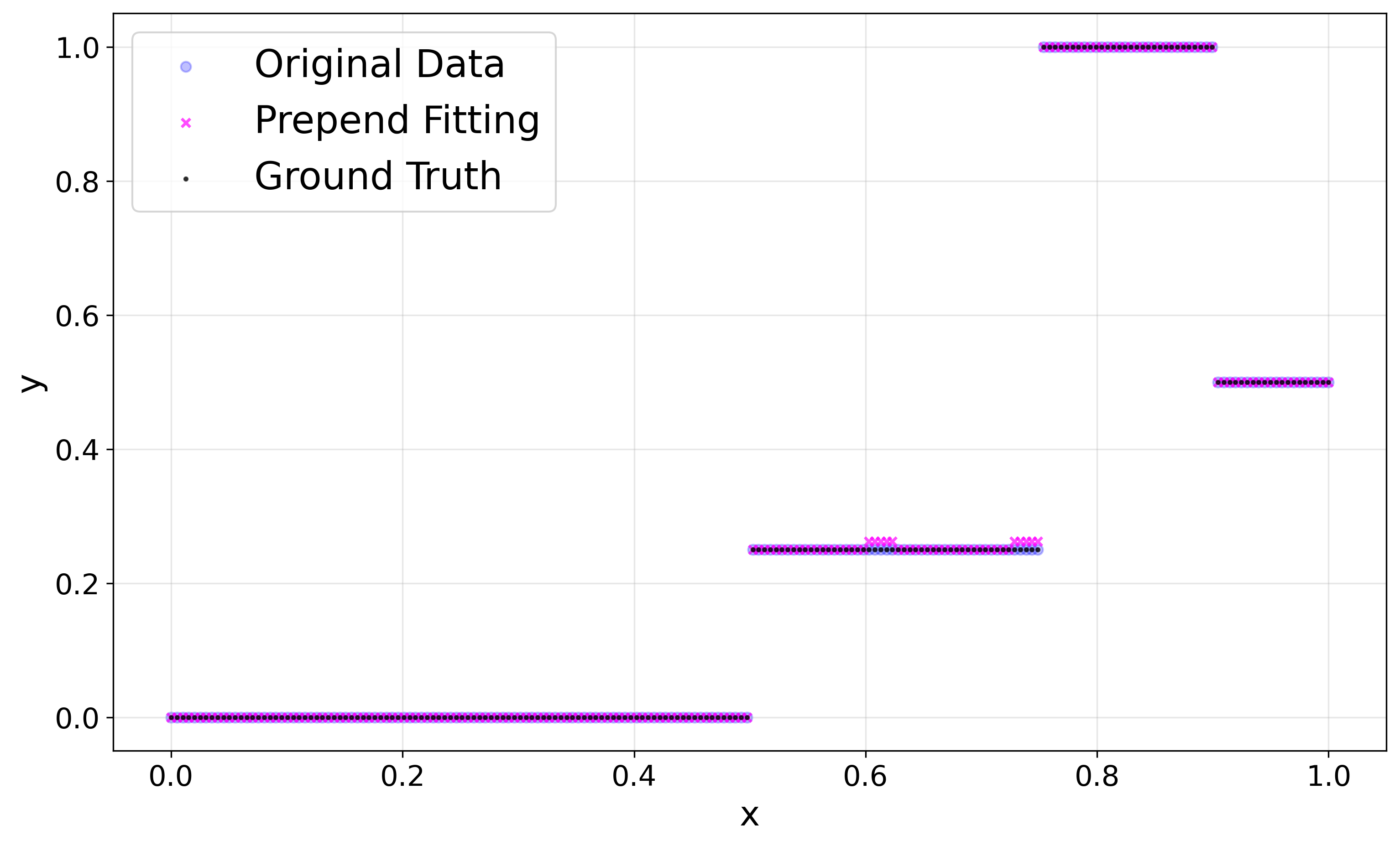}
      \includegraphics[width=0.45\linewidth]{figs/piecewise_experiment_group_prepend_only.png}
     \includegraphics[width=0.45\linewidth]{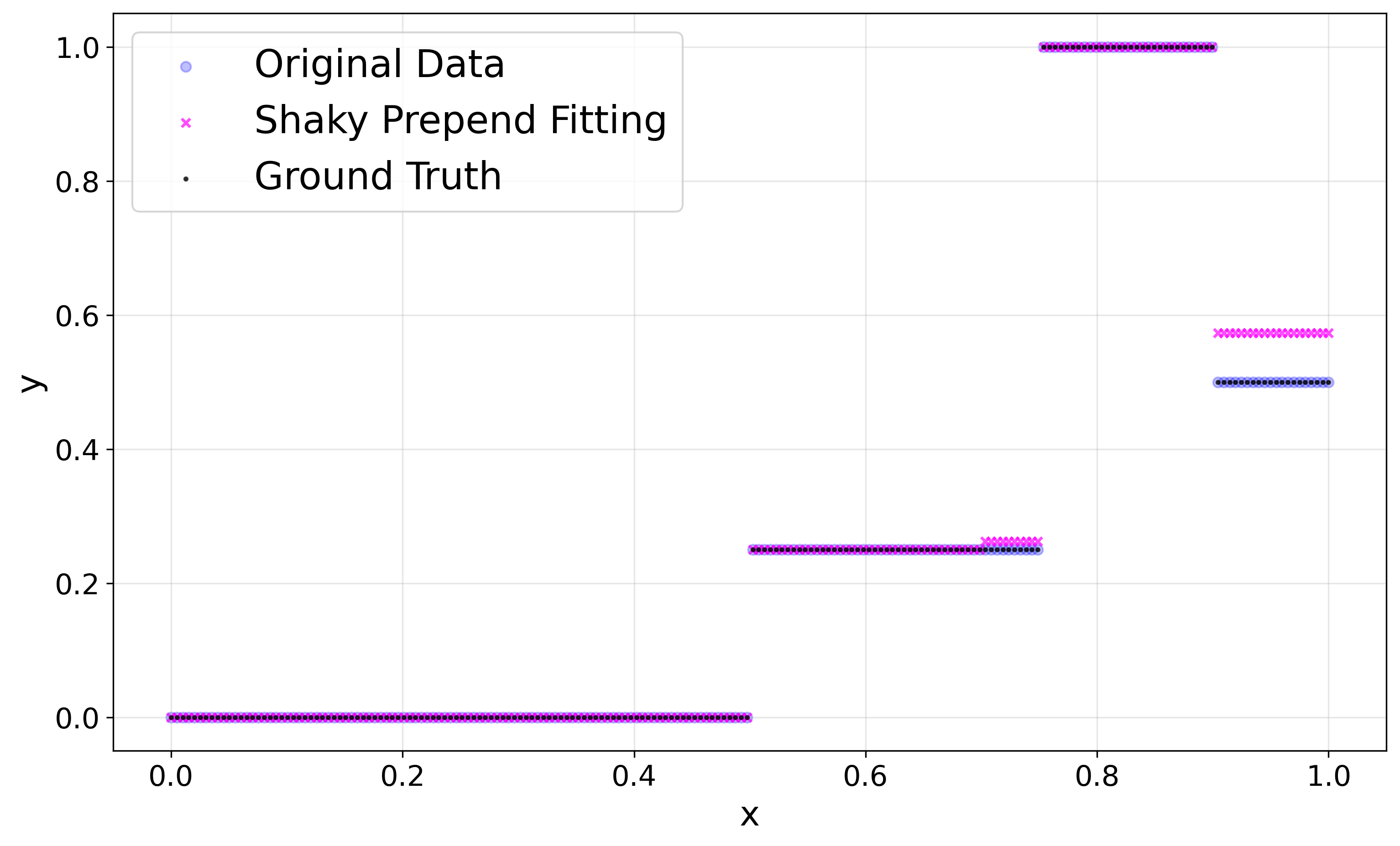}
  \includegraphics[width=0.45\linewidth]{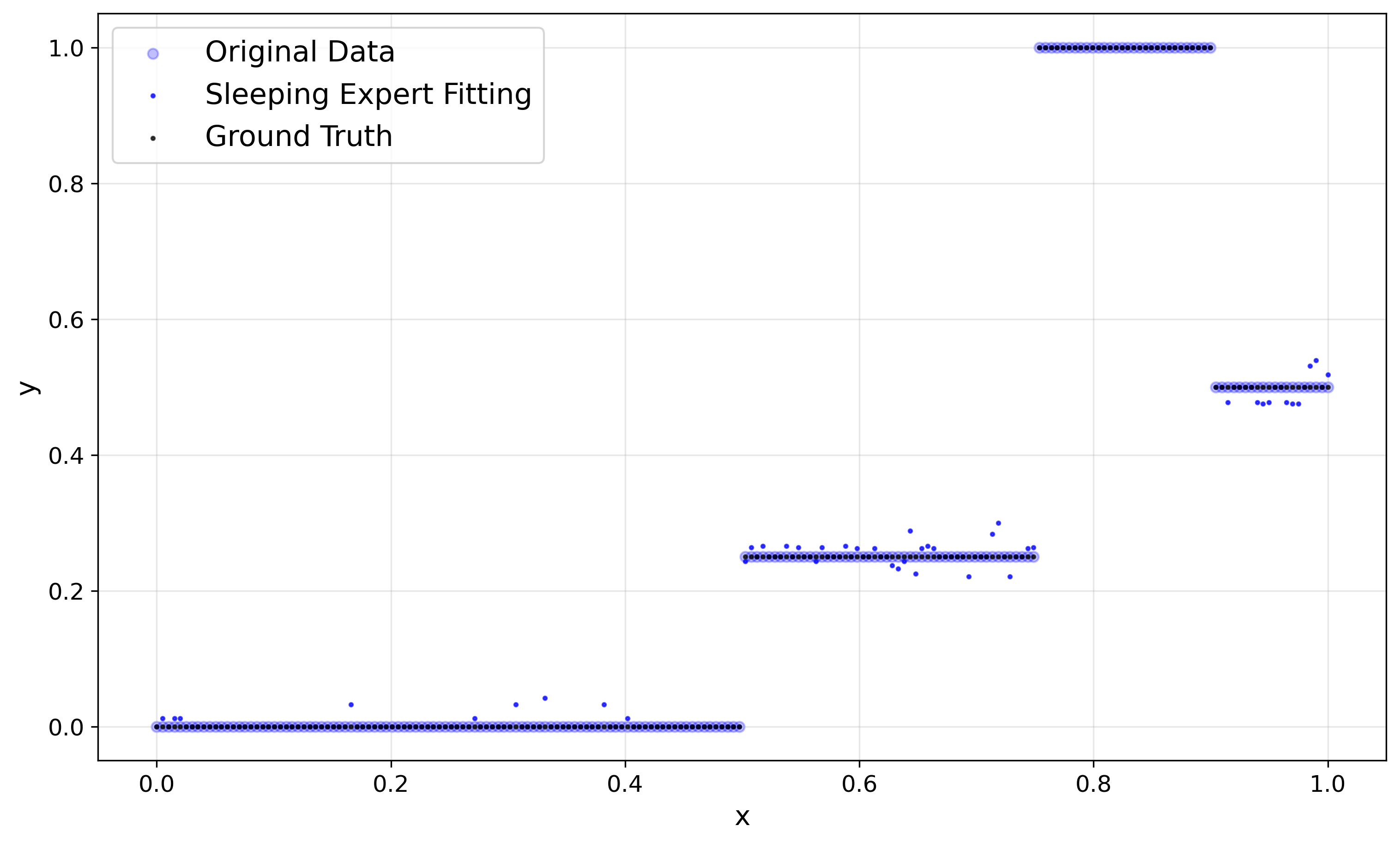}
     \caption{Predictions produced by \Prepend, \GroupPrepend, \ShakyPrepend, and \SleepingExpert (top-left to bottom-right) under no additive noise. All methods recover the target's spatial structure.}
     \label{fig:additional_plot2}
 \end{figure}
\section{Missing Proof}\label{sec:proof}

\subsection{Missing Proof in Section~\ref{sec:background}}
\par We first introduce the classical result of the Sparse algorithm:
\begin{theorem}[Theorem 3.25 in \citealp*{dwork2014algorithmic}]\label{DP_original}
    Sparse is $(\epsilon,\delta)$-differentially private.
\end{theorem}
\begin{algorithm}[ht]
\caption{Sparse}
\label{alg:sparse}
\begin{algorithmic}
    \REQUIRE{Private database $D$, an adaptively chosen stream of 1-sensitive queries $f_1, \ldots$, 
a threshold $\lambda$, and a cutoff $c$.}
\STATE{
    $\sigma \gets \frac{2c}{\epsilon}$ if $\delta=0$ else $\frac{\sqrt{32 c \ln \frac{1}{\delta}}}{\epsilon}$}
\STATE{$\widehat{\lambda}_0 \gets \lambda + \mathrm{Lap}(\sigma)$\;
$count \gets 0$\;}
\FOR{each query $i$}
\STATE{$\nu_i=\text{Lap}(2\sigma)$}
\IF{$f_i(D)+\nu_i\geq\widehat{\lambda}_{\text{count}}$}
\STATE{Output $a_i=\top$}
\STATE{$\text{count}=\text{count}+1$}
\STATE{$\widehat{\lambda}_{\text{count}}=\lambda+\text{Lap}(\sigma)$}
\ELSE\STATE{Output $a_i=\bot$}
\ENDIF
\IF{$\text{count}\geq c$}
\STATE{Halt.}
\ENDIF
\ENDFOR
\end{algorithmic}
\end{algorithm}
\begin{proof}[Proof of Theorem~\ref{thm:DP_adjusted}]
If we set $c=k$, $\sigma=\frac{\Delta\sqrt{32\ln(1/\delta)}}{\epsilon}$, and assume each $f_i$ in Algorithm~\ref{alg:sparse} is a $\Delta$-sensitive query, then the threshold test
\[
f_i(D)+\nu_i \;\ge\; \widehat{\lambda}_{\text{count}}
\]
is equivalent, after rescaling by $\Delta$, to
\[
\frac{1}{\Delta}f_i(D) + \nu_i' \;\ge\; \widehat{\lambda}_{\text{count}}',
\]
where $\frac{1}{\Delta}f_i(D)$ has sensitivity $1$ and $\nu_i' \sim \mathrm{Lap}\!\left(2\frac{\sqrt{32\ln(1/\delta)}}{\epsilon}\right)$.
Thus, the rescaled procedure is precisely Algorithm~\ref{alg:sparse} applied to $1$-sensitive queries with an effective privacy parameter $\epsilon'=\epsilon\sqrt{k}$, and therefore the resulting mechanism is $(\epsilon\sqrt{k},\delta)$-DP by Theorem~\ref{DP_original}.

The only difference between Algorithm~\ref{alg:adjusted_sparse} and the above procedure lies in the stopping rule.
Let $\mathcal{Y}$ denote the output space, i.e., the set of all possible sequences $\{a_i:i=1,2,\ldots\}$ produced before termination. For any $y\in\mathcal{Y}$, let $l(y)$ be the number of $\top$ symbols in $y$, which equals the total number of updates. For an integer $c\ge 1$, let $y_{:c}$ denote the prefix of $y$ up to (and including) the $c$-th $\top$; if $l(y)<c$, define $y_{:c}\triangleq y$.

Let $\mathcal{M}$ denote Algorithm~\ref{alg:adjusted_sparse}. Define the truncated mechanism $\mathcal{M}'$ by
\[
\mathcal{M}'(D)\;\triangleq\; \mathcal{M}(D)_{:k},
\]
where $D$ is the input dataset. Equivalently, $\mathcal{M}'$ runs $\mathcal{M}$ and truncates its output (i.e., forces termination) once the number of $\top$ symbols reaches $k$, exactly the procedure we described at first. By the argument above, $\mathcal{M}'$ is $(\epsilon\sqrt{k},\delta)$-DP.

Now consider $\mathcal{M}$. For any measurable subset of outputs $S\subseteq \mathcal{Y}$, decompose
\[
S \;=\; S_{:k}\ \cup\ S_{k+1:},
\]
where $S_{:k}\triangleq \{y\in S:\ y=y_{:k}\}$ and $S_{k+1:}\triangleq S\setminus S_{:k}$. By assumption,
\[
\Pr[\mathcal{M}(D)\in S_{k+1:}] \;\le\; \delta'.
\]
Therefore,
\begin{align*}
\Pr[\mathcal{M}(D)\in S]
&\le \Pr[\mathcal{M}(D)\in S_{:k}] + \delta' \\
&= \Pr[\mathcal{M}'(D)\in S_{:k}] + \delta' \\
&\le e^{\epsilon\sqrt{k}}\Pr[\mathcal{M}'(D')\in S_{:k}] + \delta + \delta' \\
&\le e^{\epsilon\sqrt{k}}\Pr[\mathcal{M}(D')\in S_{:k}] + \delta + \delta',
\end{align*}
for any neighboring datasets $D,D'$. By definition, $M$ is $(\epsilon\sqrt{k}, \delta+\delta')$-differentially private. This proves the lemma.
\end{proof}
\subsection{Missing Proof in Section~\ref{sec:Shaky Prepend}}
To derive an upper bound on the number of updates, we need to control the added Laplace noise; this is established in Lemma~\ref{lemma:laplace_bound}.
\begin{lemma}[Laplace Tail Bound]\label{lemma:laplace_bound}
Let $X_1,\dots,X_k \stackrel{\mathrm{i.i.d.}}{\sim} \mathrm{Lap}(\sigma)$, and define
\(
a \coloneqq \max_{i\in[k]} |X_i|.
\)
Then for any $\beta\in(0,1)$,
\(
\Pr\!\left[a \ge \sigma \ln\!\left(\frac{k}{\beta}\right)\right] \le \beta.
\)
Equivalently, for any $\epsilon \ge 0$,
\(
\Pr[a \ge \epsilon] \le k \exp\!\left(-\frac{\epsilon}{\sigma}\right).
\)
\end{lemma}
\begin{proof}
For a single Laplace random variable $X \sim \mathrm{Lap}(\sigma)$, we have for any $t \ge 0$,
\[
\Pr\!\left[|X| \ge t\sigma\right]
= 2 \int_{t\sigma}^{\infty} \frac{1}{2\sigma} e^{-r/\sigma}\, \mathrm{d}r
= \int_{t\sigma}^{\infty} \frac{1}{\sigma} e^{-r/\sigma}\, \mathrm{d}r
= \int_{t}^{\infty} e^{-u}\, \mathrm{d}u
= e^{-t}.
\]
By the union bound over $X_1,...,X_k$,
\[
\Pr\!\left[\max_{i\in[k]} |X_i| \ge t\sigma\right]
\le \sum_{i=1}^k \Pr\!\left[|X_i| \ge t\sigma\right]
= k e^{-t}.
\]
Setting $t=\ln(k/\beta)$ yields
\[
\Pr\!\left[\max_{i\in[k]} |X_i| \ge \sigma \ln\!\left(\frac{k}{\beta}\right)\right]
\le \beta,
\]
which proves the lemma.

\end{proof}
\begin{proof}[Proof of Lemma~\ref{lemma:update}]
    Denote $a'=\max_{0\leq t\leq \frac{2\alpha}{\lambda}}\{\xi_t,\xi_t'\}=-\min_{0\leq t\leq \frac{2\alpha}{\lambda}}\{\xi_t,\xi_t'\}$ (by symmetry of the Laplace distribution). By the update rule, for $1\leq t\leq\frac{2\alpha}{\lambda}$,
    \[P_n(g)(L_n(f_{t-1}|g_t)-L_n(h_t|g_t))\geq \lambda -\xi_t-\xi_t'\geq\lambda-2a'.\]
    Since $f_t$ only differs from $f_{t-1}$ on $g_t$, 
    \[L_n(f_t)-L_n(f_{t-1})=P_n(g)[L_n(h_t|g_t)-L_n(f_{t-1}|g_t)]\leq -(\lambda-2a').\]
    In other words, every update decreases $L_n(f)$ by at least $\lambda-2a'$.

    If $a'<\frac{\lambda}{4}$ and $B>\frac{2\alpha}{\lambda}$,
    \[L_n(f_{\lceil\frac{2\alpha}{\lambda}\rceil})<\alpha - \frac{2\alpha}{\lambda}(\lambda-2a')<0,\quad \text{contradicting }L_n(\cdot)\geq0.\]

    $\Rightarrow\Pr[B\leq\frac{2\alpha}{\lambda}|\alpha'<\frac{\lambda}{4}]=1$
    \begin{align*}
    \Pr[B\leq\frac{2\alpha}{\lambda}]&\geq\Pr[B\leq\frac{2\alpha}{\lambda}, a'<\frac{\lambda}{4}]\\
    &=\Pr[a'\leq\frac{\lambda}{4}]\Pr[B\leq\frac{2\alpha}{\lambda}|a'<\frac{\lambda}{4}]\\
    &=\Pr[a'\leq\frac{\lambda}{4}]\geq1-e^{-\frac{\lambda}{4\sigma}}\frac{4\alpha}{\lambda}.
    \end{align*}
    The last inequality results from Lemma \ref{lemma:laplace_bound}.
\end{proof}
To derive the DP property of Algorithm~\ref{alg:shakyprepend}, we also need to analyze the sensitivity of each query:
\begin{lemma}\label{lemma:tech}
    For any $g$ and $h$, $P_n(g)(L_n(f|g)-L_n(h|g))$ is $\frac{4}{n}$-sensitive.
\end{lemma}
\begin{proof}
    Denote $P_{S_n}(g)$ to be the sample estimate of $P(g)$ on the dataset $S_n$, $L_{S_n}(f|g)$ to be the sample estimate of $L(f|g)$ on the dataset $S_n$. WLOG, suppose the adjacent dataset of $S_n$ to be $S_n'=\{(x_1',y_1'),(x_2,y_2),...,(x_n,y_n)\}$.

    If $g(x_1)=g(x_1')$, 
    \begin{align*}
        &|P_{S_n}(g)(L_{S_n}(f|g)-L_{S_n}(h|g))-P_{S_n'}(g)(L_{S_n'}(f|g)-L_{S_n'}(h|g))|\\
        &=P_{S_n}(g)\frac{1}{nP_{S_n}(g)}|l(f(x_1),y_1)-l(f(x_1'),y_1')-l(h(x_1),y_1)+l(h(x_1'),y_1')|\leq\frac{2}{n},
    \end{align*}
    If $g(x_1)>g(x_1')$,
    \begin{align*}
        &|P_{S_n}(g)(L_{S_n}(f|g)-L_{S_n}(h|g))-P_{S_n'}(g)(L_{S_n'}(f|g)-L_{S_n'}(h|g))|\\
        &=|P_{S_n}(g)(L_{S_n}(f|g)-L_{S_n}(h|g)-L_{S_n'}(f|g)+L_{S_n'}(h|g))|\\
        &+|(L_{S_n'}(f|g)-L_{S_n'}(h|g))(P_{S_n}(g)-P_{S_n'}(g))|\triangleq A+B.\\
    \end{align*}
    \[\left| L_{S_n}(f \mid g) - L_{S_n'}(f \mid g) \right|=|\frac{\ell(x_1,y_1)+\sum_{i=2}^n\ell(x_i,y_i)}{nP_{S_n}(g)}-\frac{\sum_{i=2}^n\ell(x_i,y_i)}{nP_{S_n}(g)-1}{}|\leq\frac{1}{nP_{S_n}(g)}.\]
    Thus,
    \begin{align*}
        &A\leq P_{S_n}(g)(\frac{1}{nP_{S_n}(g)}+\frac{1}{nP_{S_n}(g)})=\frac{2}{n},\\
        &B\leq (|L_{S_n'}(f|g)|+|L_{S_n'}(h|g)|)|P_{S_n}(g)-P_{S_n'}(g)|\leq 2*\frac{1}{n}=\frac{2}{n},\\
        &|P_{S_n}(g)(L_{S_n}(f|g)-L_{S_n}(h|g))-P_{S_n'}(g)(L_{S_n'}(f|g)-L_{S_n'}(h|g))|\leq\frac{4}{n}.
    \end{align*}
    When $g(x_1)<g(x_1')$, the same conclusion follows as when $g(x_1)>g(x_1').$
\end{proof}
We are now ready to prove the DP property of Algorithm~\ref{alg:shakyprepend}.
\begin{proof}[Proof of Lemma~\ref{lemma:DP_of_shaky}]
\par The algorithm can be viewed as a game:
\begin{itemize}
    \item For initialization, sample $\xi_0\sim\text{Lap}(\sigma)$,$\lambda_0=\lambda+\xi_0, \text{count}=0$.
    \item Every time the adversary query $q_i:\mathcal{X}^n\to\mathbb{R}, q_i(S_n)=P_n(g_i)[L_n(f_i|g_i)-L_n(h_i|g_i)]$ according to the rule 
    \begin{itemize}
        \item If last response is yes, $f_i=[g_i,h_i,f_{i-1}]$ and $(g_i,h_i)$ starts over to enumerate $\mathcal{G}\times\mathcal{H}$
        \item If last response is no, $f_i=f_{i-1}$, $(g_i,h_i)$ continues to enumerate $\mathcal{G}\times\mathcal{H}$
        \item If last response is no and $(g_i,h_i)$ has finished enumerating, output $f_i$ and end the game.
    \end{itemize}
    \item For the mechanism, for each query, sample $\mu_i\sim\text{Lap}(2\sigma)$, if $q_i(S_n)+\mu_i=L_n(f_i|g_i)-L_n(h_i|g_i)+\mu_i\geq \lambda_{\text{count}}$, respond yes, else respond no, $\text{count} =\text{count}+1$, sample $\xi_{\text{count}}\sim\text{Lap}(\sigma)$, and let $\lambda_{\text{count}}=\lambda+\xi_{\text{count}}$.
\end{itemize}
\par The mechanism can be fit into Algorithm~\ref{alg:adjusted_sparse}. By Lemma ~\ref{lemma:tech} and according to Theorem~\ref{thm:DP_adjusted}, the mechanism is $(\epsilon\sqrt{\frac{2\alpha}{\lambda}},\delta+e^{-\frac{\lambda}{4\sigma}}\frac{4\alpha}{\lambda})$-differentially private.
\end{proof}

\begin{proof}[Proof of Lemma~\ref{lemma:gen_bound}]
\par We cite a related theorem in \citet{BNS+} here: (Note $(\epsilon,\delta)-max-KL-stable$ is the same as $(\epsilon,\delta)$-DP).
\begin{theorem}[Theorem 7.2 in \citealp*{BNS+}]\label{thm:generalization}
Let $\varepsilon \in (0,1/3)$, $\delta \in (0,\varepsilon/4)$, and suppose $n\geq\frac{1}{\varepsilon^{2}}\ln\!\Bigl(\tfrac{4\varepsilon}{\delta}\Bigr).$
Let $\mathcal{M}\colon \mathcal{X}^{n} \to \mathcal{Q}_{\Delta}$ be an
$(\varepsilon,\delta)$‑max‑

\noindent KL–stable mechanism, where
$\mathcal{Q}_{\Delta}$ is the class of $\Delta$‑sensitive queries
$q\colon \mathcal{X}^{n} \to \mathbb{R}$.
Let $P$ be a distribution over $\mathcal{X}$,
draw $\mathbf{x} \leftarrow_{\!R} P^{n}$,
and sample $q \leftarrow_{\!R} \mathcal{M}(\mathbf{x})$.
Then
\[
  \Pr_{\mathbf{x},\,\mathcal{M}}\!\left[
    \bigl| q(P) - q(\mathbf{x}) \bigr|
    \;\ge\; 18\,\varepsilon\,\Delta\,n
  \right]
  \;<\; \frac{\delta}{\varepsilon}.
\]
\end{theorem}
After processing algorithm \ref{alg:shakyprepend}, the adversary can post-process its final function $f_B$ to get a statistical query. Altogether, the process can be seen as a mapping from the dataset to a query, thus fitting into the above cited theorem. We set $\mathcal{M}_1(S_n)=\mathcal{X}^n\to\frac{1}{n}\sum_{i=1}^n g(x_i)f_{S_n}(x_i)$, $\mathcal{M}_2(S_n)=\mathcal{X}^n\to\frac{1}{n}\sum_{i=1}^n g(x_i)$. Obviously $\frac{1}{n}\sum_{i=1}^ng(x_i)$ and $\frac{1}{n}\sum_{i=1}^ng(x_i)f_{S_n}(x_i)$ are $\frac{1}{n}$-sensitive. Since the post-processing maintains the differential-privacy property and the Algorithm \ref{alg:shakyprepend} is $(\epsilon\sqrt{\frac{2\alpha}{\lambda}},\delta+e^{-\frac{\lambda}{4\sigma}}\frac{4\alpha}{\lambda})$-differentially private according to Lemma~\ref{lemma:DP_of_shaky}, the mechanism $\mathcal{M}_1$ and $\mathcal{M}_2$ should also be $(\epsilon\sqrt{\frac{2\alpha}{\lambda}}, \delta+e^{-\frac{\lambda}{4\sigma}}\frac{4\alpha}{\lambda})$-DP to $\frac{1}{n}$-sensitive queries. By Theorem~\ref{thm:generalization}, set $\epsilon'=\epsilon\sqrt{\frac{2\alpha}{\lambda}}$ and $\delta'=\delta+e^{-\frac{\lambda}{4\sigma}}\frac{4\alpha}{\lambda}$, when $n\geq\frac{1}{\epsilon^2\frac{2\alpha}{\lambda}}\ln(\frac{4\epsilon\sqrt{\frac{2\alpha}{\lambda}}}{\delta+e^{-\frac{\lambda}{4\sigma}}\frac{4\alpha}{\lambda}})$,
\[\Pr_{\mathbf{x}}[|P_n(g)-P(g)|\geq18\epsilon\sqrt{\frac{2\alpha}{\lambda}}]<\frac{\delta+e^{-\frac{\lambda}{4\sigma}}\frac{4\alpha}{\lambda}}{\epsilon\sqrt{\frac{2\alpha}{\lambda}}}\leq\frac{\delta+e^{-\frac{\lambda}{4\sigma}}\frac{4\alpha}{\lambda}}{\epsilon},\]
\[\Pr_{\mathbf{x}}[|P_n(g)L_n(f_{S_n}|g)-P(g)L(f_{S_n}|g)|\geq18\epsilon\sqrt{\frac{2\alpha}{\lambda}}]<\frac{\delta}{\epsilon\sqrt{\frac{2\alpha}{\lambda}}}\leq\frac{\delta}{\epsilon}.\]

At the same time, 
\begin{align*}
    &|L_n(f_{S_n}|g)-L(f_{S_n}|g)|\\
    &\leq\frac{1}{P_n(g)}[|P_n(g)L_n(f_{S_n}|g)-P(g)L(f_{S_n}|g)|+|P_n(g)-P(g)|L(f_{S_n}|g)]\\
    &\leq\frac{1}{P_n(g)}[|P_n(g)L_n(f_{S_n}|g)-P(g)L(f_{S_n}|g)|+|P_n(g)-P(g)|].
\end{align*}

Thus,
\begin{align*}
  &\Pr[|L_n(f_{S_n}|g)-L(f_{S_n}|g)|\geq\frac{36\epsilon\sqrt{\frac{2\alpha}{\lambda}}}{P_n(g)}]\\
  &\leq\Pr[|P_n(g)-P(g)|\geq9\epsilon\sqrt{\frac{2\alpha}{\lambda}}]+\Pr[|P_n(g)L_n(f_{S_n}|g)-P(g)L(f_{S_n}|g)|\geq9\epsilon\sqrt{\frac{2\alpha}{\lambda}}]\\
  &<\frac{2(\delta+e^{-\frac{\lambda}{4\sigma}}\frac{4\alpha}{\lambda})}{\epsilon}.
\end{align*}
Construct similar $\mathcal{M}_1$ and $\mathcal{M}_2$ for every $g\in\mathcal{G}$. Taking a union bound over $g\in\mathcal{G}$,
\[\Pr[\max_{g\in\mathcal{G}}|(L_n(f_{S_n}|g)-(L(f_{S_n}|g)|>\frac{36\epsilon\sqrt{\frac{2\alpha}{\lambda}}}{P_n(g)}]<\frac{2|\mathcal{G}|(\delta+e^{-\frac{\lambda}{4\sigma}}\frac{4\alpha}{\lambda})}{\epsilon}.\]
\end{proof}
We first state a simplified finite-class version of Theorem~\ref{thm:Cond_Emp_Risk}, and then prove Lemma~\ref{lemma:err_bound}.
\begin{corollary}\label{cor:Cond_Emp_Risk}
If $\max\{|\mathcal{H}|,|\mathcal{G}|\}<\infty$, $\forall (h,g)\in\mathcal{H}\times\mathcal{G}$, with probability at least $1-\delta$,
\[
\left| L(h \mid g) - L_n(h \mid g) \right| \le 9 \sqrt{\frac{2\ln(|\mathcal{G}||\mathcal{H}|)+\ln(8/\delta)}{nP_n(g)}}.
\]
\end{corollary}
\begin{proof}[Proof of Lemma~\ref{lemma:err_bound}]
    \par We decompose the object into three parts and bound it above by the sum of their absolute values:
    \begin{align*}       
    |L(f_B|g)-\min_{h\in\mathcal{H}}L(h|g)|
    \leq |L_n(f_B|g)-L(f_B|g)|+|L_n(f_B|g)-\min_{h\in\mathcal{H}}L_n(h|g)|+|L_n(h|g)-L(h|g)|.
    \end{align*}
    Since there is no update during the last iteration, we have for any $g\in\mathcal{G}$,
    \[P_n(g)[L_n(f_B|g)-\min_hL_n(h|g)]\leq\lambda+\max_h-\mu_{B,g,h}+\xi_B\leq\lambda+a+\xi_B,\]
    
    By corollary \ref{cor:Cond_Emp_Risk} and set $\delta=8(|\mathcal{H}||\mathcal{G}|)^2e^{-\frac{2\epsilon^2 n}{81\lambda}}$, we have
    \[\Pr[\max_{(h,g)}|L_n(h|g)-L(h|g)|\geq \epsilon\sqrt{\frac{2}{\lambda P_n(g)}}]\leq 8(|\mathcal{H}||\mathcal{G}|)^2e^{-\frac{2\epsilon^2 n}{81\lambda}}.\]
    Combining with lemma \ref{lemma:gen_bound}, we have that
    \begin{small}
        \[\Pr[\max_{g\in\mathcal{G}}|L(f_B|g)-\min_{h\in\mathcal{H}}L(h|g)|>\frac{36\epsilon\sqrt{\frac{2\alpha}{\lambda}}+\lambda+a+\xi_B}{P_n(g)}+\epsilon\sqrt{\frac{2}{\lambda P_n(g)}}]\leq\frac{2|\mathcal{G}|(\delta+e^{-\frac{\lambda}{4\sigma}}\frac{4\alpha}{\lambda})}{\epsilon}+8e^{-\frac{2\epsilon^2 n}{81\lambda}}|\mathcal{H}|^2|\mathcal{G}|^2.\]
    \end{small}
    Also, $\frac{1}{P_n(g)}\geq\frac{1}{\sqrt{P_n(g)}}$ when $P_n(g)\leq1$, so 
    \begin{small}
      \[\Pr[\max_{g\in\mathcal{G}}|L(f_B|g)-\min_{h\in\mathcal{H}}L(h|g)|>\frac{36\epsilon\sqrt{\frac{2\alpha}{\lambda}}+\lambda+a+\xi_B+\epsilon\sqrt{\frac{2}{\lambda}}}{P_n(g)}]\leq\frac{|\mathcal{G}|(\delta+e^{-\frac{\lambda}{4\sigma}}\frac{4\alpha}{\lambda})}{\epsilon}+8|\mathcal{H}|^2|\mathcal{G}|^2e^{-\frac{2\epsilon^2 n}{81\lambda}}.\]      
    \end{small}
\end{proof}
Now we can prove the main theorem.
\begin{proof}[Proof of Theorem~\ref{thm:shakyprepend}]
    Since $\ell(x,y)\leq1$ for all $x\in\mathcal{X},y\in\mathcal{Y}$, we can set $\alpha=1$ in Lemma~\ref{lemma:update}.
    Recall the definition in the lemmas: $B$ is the number of updates before stopping, $a$ is the maximum of the opposite Laplace random variable sampled during the last round, say $a=\max_{g,h}-\mu_{B,g,h}$. We first assume that $n\geq\frac{\lambda}{2\epsilon^2}\ln(\frac{4\epsilon\sqrt{\frac{2}{\lambda}}}{\delta})$ so that we can apply Lemma~\ref{lemma:err_bound}. 
     \begin{align*}    
     &\Pr[\{L(f|g)-\min_hL(h|g)\}\geq \frac{1}{P_n(g)}(37\epsilon\sqrt{\frac{2}{\lambda}}+\lambda + 2\ln(\frac{4|\mathcal{G}||\mathcal{H}|}{\beta})\sigma)]\\
        \leq& \Pr[\{L(f|g)-\min_hL(h|g)\}\geq\frac{36\epsilon\sqrt{\frac{2}{\lambda}}+\lambda+a+\xi_B+\epsilon\sqrt{\frac{2}{\lambda}}}{P_n(g)}]+\Pr[a\geq\ln(\frac{4|\mathcal{G}||\mathcal{H}|}{\beta})\sigma] \\
        &+ \Pr[\xi_B\geq\ln(\frac{4|\mathcal{G}||\mathcal{H}|}{\beta})\sigma]\\
        \leq& \frac{2|\mathcal{G}|(\delta+e^{-\frac{\lambda}{4\sigma}}\frac{4}{\lambda})}{\epsilon}+8|\mathcal{H}|^2|\mathcal{G}|^2e^{-\frac{2\epsilon^2 n}{81\lambda}}+2\beta. \quad \text{(Lemma~\ref{lemma:err_bound}, Lemma~\ref{lemma:laplace_bound})}
    \end{align*}
    Recall $\sigma=4\sqrt{32\ln(1/\delta)}/(n\epsilon)$, set
    \[\delta=\frac{\beta}{2n|\mathcal{G}||\mathcal{H}|},\quad \epsilon = n^{-\frac{3}{5}}(\ln\frac{4|\mathcal{G}||\mathcal{H}|}{\beta}\sqrt{\ln\frac{1}{\delta}})^\frac{3}{5}=n^{-\frac{3}{5}}(\ln\frac{4|\mathcal{G}||\mathcal{H}|}{\beta}\sqrt{\ln\frac{2n|\mathcal{G}||\mathcal{H}|}{\beta}})^\frac{3}{5},\]
    \[\lambda=16\sqrt{32}\epsilon^{2/3}=16\sqrt{32}n^{-\frac{2}{5}}(\ln\frac{4|\mathcal{G}||\mathcal{H}|}{\beta}\sqrt{\ln\frac{2n|\mathcal{G}||\mathcal{H}|}{\beta}})^{\frac{2}{5}}(\ln n)^{\frac{2}{5}}.\]
    We can first check that $n\geq\frac{\lambda}{2\epsilon^2}\ln(\frac{4\epsilon\sqrt{\frac{2\alpha}{\lambda}}}{\delta})$, because $\frac{\lambda}{2\epsilon^2}\ln(\frac{4\epsilon\sqrt{\frac{2}{\lambda}}}{\delta})\sim n^{\frac{4}{5}}$. 
    At the same time,
    \[36\epsilon\sqrt{\frac{2}{\lambda}}\sim n^{-\frac{2}{5}}(\ln\frac{4|\mathcal{G}||\mathcal{H}|}{\beta}\sqrt{\ln\frac{2n|\mathcal{G}||\mathcal{H}|}{\beta}})^{\frac{2}{5}}(\ln n)^{\frac{2}{5}},\]
    \[\ln(\frac{4|\mathcal{G}||\mathcal{H}|}{\beta})\sigma\sim n^{\frac{2}{5}}(\ln\frac{4|\mathcal{G}||\mathcal{H}|}{\beta}\sqrt{\ln\frac{2n|\mathcal{G}||\mathcal{H}|}{\beta}})^{\frac{2}{5}}(\ln n)^{-\frac{3}{5}}.\]
    Thus,
    \[\lambda + 36\epsilon\sqrt{\frac{2}{\lambda}}+\ln(\frac{4|\mathcal{G}||\mathcal{H}|}{\beta})\sigma \sim n^{-\frac{2}{5}}(\ln\frac{4|\mathcal{G}||\mathcal{H}|}{\beta}\sqrt{\ln\frac{2n|\mathcal{G}||\mathcal{H}|}{\beta}})^{\frac{2}{5}}(\ln n)^{\frac{2}{5}}.\]

    By technical lemma \ref{lemma:tech_2}, $\epsilon\geq\frac{1}{n}, \lambda\epsilon\geq\frac{1}{n}.$
    \par On the other hand,
    \(\frac{2|\mathcal{G}|\delta}{\epsilon}=\frac{n\beta}{\epsilon|\mathcal{H}|}\leq\beta,\quad(\text{using }\epsilon\geq\frac{1}{n})\)
    \begin{align*}8|\mathcal{H}|^2|\mathcal{G}|^2e^{-\frac{2\epsilon^2n}{81\lambda}}&=8|\mathcal{H}|^2|\mathcal{G}|^2e^{-2n\lambda^2}\\
    &\leq8|\mathcal{H}|^2|\mathcal{G}|^2e^{-2n^{1/5}(\ln\frac{4|\mathcal{G}||\mathcal{H}|}{\beta}\sqrt{\ln\frac{2n|\mathcal{G}||\mathcal{H}|}{\beta}})^{\frac{4}{5}}(\ln n)^{\frac{4}{5}}}\\
    &\leq8|\mathcal{H}|^2|\mathcal{G}|^2e^{-2(\ln(\frac{2|\mathcal{G}||\mathcal{H}|}{\beta})\sqrt{\ln(\frac{2|\mathcal{G}||\mathcal{H}|}{\beta}})^{\frac{4}{5}}}\\
    &\leq 8|\mathcal{H}|^2|\mathcal{G}|^2e^{-2\ln(\frac{2|\mathcal{G}||\mathcal{H}|}{\beta})}=2\beta^2\leq 2\beta,
    \end{align*}
    \[\frac{\lambda}{4\sigma}=\frac{\epsilon^{2/3}n\epsilon}{\sqrt{\ln(1/\delta)}}=\frac{\epsilon^{5/3}n}{\sqrt{\ln(1/\delta)}}=\ln n\ln(\frac{4|\mathcal{G}||\mathcal{H}|}{\beta})\geq\ln n+\ln(\frac{4|\mathcal{G}||\mathcal{H}|}{\beta})\]
    \[\frac{2|\mathcal{G}|}{\epsilon}e^{-\frac{\lambda}{4\sigma}}\frac{4}{\lambda}\leq\frac{2\beta}{n\epsilon\lambda|\mathcal{H}|}\leq \beta,\quad(\text{using }\lambda\epsilon\geq \frac{1}{n} )\]
    Thus, 
    \[\frac{2|\mathcal{G}|(\delta+e^{-\frac{\lambda}{4\sigma}}\frac{4}{\lambda})}{\epsilon}+8|\mathcal{H}|^2|\mathcal{G}|^2e^{-\frac{2\epsilon^2 n}{81\lambda}}+2\beta\leq 6\beta.\]
    Put them together and rescale $\beta$, for any $g\in\mathcal{G}$,
    \[\Pr[\{L(f|g)-\min_{h\in\mathcal{H}}L(h|g)\}\gtrsim \frac{1}{P_n(g)}n^{-\frac{2}{5}}\ln(\frac{24|\mathcal{G}||\mathcal{H}|}{\beta})^{\frac{2}{5}}\ln(\frac{1}{\delta})^{\frac{1}{5}}]\leq\beta.\]
\end{proof}
\begin{lemma}\label{lemma:tech_2} When
\[\epsilon = n^{-\frac{3}{5}}(\ln\frac{4|\mathcal{G}||\mathcal{H}|}{\beta}\sqrt{\ln\frac{1}{\delta}})^\frac{3}{5}=n^{-\frac{3}{5}}(\ln\frac{4|\mathcal{G}||\mathcal{H}|}{\beta}\sqrt{\ln\frac{2n|\mathcal{G}||\mathcal{H}|}{\beta}})^\frac{3}{5},\]
    \[\lambda=16\sqrt{32}\epsilon^{2/3}=16\sqrt{32}n^{-\frac{2}{5}}(\ln\frac{4|\mathcal{G}||\mathcal{H}|}{\beta}\sqrt{\ln\frac{2n|\mathcal{G}||\mathcal{H}|}{\beta}})^{\frac{2}{5}}(\ln n)^{\frac{2}{5}},\]
    we have $\epsilon\geq\frac{1}{n},\lambda\epsilon\geq\frac{1}{n}.$
\end{lemma}
\begin{proof}
    Since $|\mathcal{G}|\geq1$, $|\mathcal{H}|\geq 1$, $\beta\leq 1$, $n\geq 2$,
    \[\epsilon \geq n^{-\frac{3}{5}}(\ln 4\times \sqrt{\ln 2})^{\frac{3}{5}}>n^{-\frac{3}{5}}\geq \frac{1}{n},\]
    \[\epsilon\lambda =16\sqrt{32}n^{-1}(\ln 4\times\sqrt{\ln 2})(\ln 2)^{\frac{2}{5}}\geq\frac{1}{n}.\]
\end{proof}
\subsection{Fractional-version of Prepend-like Algorithms}\label{sec:fraction}
\begin{algorithm}[ht]
\caption{\FractionalShakyPrepend}
\label{alg:smoothedshakyprepend}
\begin{algorithmic}
\REQUIRE Groups $\mathcal{G}$, hypothesis class $\mathcal{H}$, i.i.d.\ examples $(x_1, y_1), \ldots, (x_n, y_n)$ from $\mathcal{D}$, error bound $\lambda$, step size $\eta$, privacy parameters $(\epsilon,\delta)$.
\STATE Compute $h_0 \in \arg\min_{h \in \mathcal{H}} L_n(h)$
\STATE Set $\sigma \gets 4\sqrt{32\ln(1/\delta)}/(n\epsilon)$ and $f_0 \gets [1,h_0]$
\STATE Sample $\xi_0\sim\text{Lap}(\sigma)$ and set $\lambda_0 \gets \lambda + \xi_0$
\FOR{$t = 0, 1, \ldots$}
    \STATE update $\gets$ False
    \FOR{$(g,h)\in\mathcal{G}\times\mathcal{H}$}
        \STATE Sample $\mu_{t,g,h}\sim \text{Lap}(2\sigma)$
        \STATE Define $f'(x) \gets f_t(x)+ \eta\, g(x)\bigl(h(x)-f_t(x)\bigr)$
        \IF{$P_n(g)\bigl(L_n(f_t \mid g) - L_n(f' \mid g)\bigr)+\mu_{t,g,h} \geq \lambda_t$}
            \STATE $(g_{t+1},h_{t+1}) \gets (g,h)$;\quad $\xi_t' \gets \mu_{t,g,h}$
            \STATE $f_{t+1}(x) \gets f'(x)$
            \STATE update $\gets$ True
            \STATE \textbf{break}
        \ENDIF
    \ENDFOR
    \IF{update}
        \STATE Sample $\xi_{t+1}\sim\text{Lap}(\sigma)$ and set $\lambda_{t+1} \gets \lambda+\xi_{t+1}$
    \ELSE
        \STATE Return $f_t$
    \ENDIF
\ENDFOR
\end{algorithmic}
\end{algorithm}
\begin{proof}[Proof of Theorem~\ref{alg:smoothedshakyprepend}]
According to the update rules we have:
    \[P_n(g)(L_n(f_B|g)-L_n(f_B+\eta(h-f_B)|g))\geq \lambda_B - \mu_{B,g,h}.\]
By Jensen's inequality,
\[L_n(f_B+\eta(h-f_B)|g)\leq (1-\eta)L_n(f_B|g)+\eta L+n(h|g).\]
Thus,
\[P_n(g)(L_n(f_B|g)-L_n(h|g))\geq\frac{\lambda-\mu_{B,g,h}}{\eta}.\]

The differential-privacy lemma and the generalization lemma still hold. Thus, the error bound lemma becomes:
\begin{small}
    \[\Pr[\max_{g\in\mathcal{G}}|L(f_B|g)-\min_{h\in\mathcal{H}}L(h|g)|>\frac{36\epsilon\sqrt{\frac{2\alpha}{\lambda}}+\frac{\lambda}{\eta}+\frac{a+\xi_B}{\eta}+\epsilon\sqrt{\frac{2}{\lambda}}}{P_n(g)}]\leq\frac{2|\mathcal{G}|(\delta+e^{-\frac{\lambda}{4\sigma}}\frac{4}{\lambda})}{\epsilon}+8|\mathcal{H}|^2|\mathcal{G}|^2e^{-\frac{2\epsilon^2 n}{81\lambda}}.\]
\end{small}
By picking appropriate $\lambda$, $\epsilon$, we have for all $g\in\mathcal{G}$, with probability of at least $1-\delta$,
\[
L(f\mid g)-\min_h L(h\mid g)
\lesssim
\frac{n^{-\frac{2}{5}}\ln n}{\eta P_n(g)}\,
\ln\!\left(\frac{4|\mathcal{G}||\mathcal{H}|}{\beta}\right)^{\frac{2}{5}}
\ln\!\left(\frac{1}{\delta}\right)^{\frac{1}{5}}.
\]
\end{proof}
\subsection{\GroupPrepend}
\begin{algorithm}[ht]
\caption{\GroupPrepend}
\label{alg:groupprepend}
\begin{algorithmic}
\REQUIRE Groups $\mathcal{G}$, hypothesis class $\mathcal{H}$, i.i.d.\ examples from $\mathcal{D}$, error bound $\lambda$.
\STATE Compute $h_0 \in \arg\min_{h \in \mathcal{H}} L_n(h)$. Set $f_0 = [h_0] \in \mathrm{DL}_0[\mathcal{G}; \mathcal{H}]$
\FOR{$t = 0, 1, \ldots$}
    \STATE Compute
    \[
    (g_{t+1}, h_{t+1}) \in \arg\max_{(g, h) \in \mathcal{G} \times \mathcal{H}} P_n(g)(L_n(f_t \mid g) - L_n(h \mid g))
    \]
    \IF{ $P_n(g)(L_n(f_t \mid g_{t+1}) - L_n(h_{t+1} \mid g_{t+1})) \geq \lambda$ }
        \STATE Prepend $(g_{t+1}, h_{t+1})$ to $f_t$ to obtain
        \[
        f_{t+1} := [g_{t+1}, h_{t+1}, g_t, h_t, \ldots, g_1, h_1, h_0]
        \]
    \ELSE
        \STATE Return $f_t$.
    \ENDIF
\ENDFOR
\end{algorithmic}
\end{algorithm}
The theoretical guarantee of the \GroupPrepend algorithm is stated and proved below.

\begin{theorem}
    Set $\lambda = n^{-\frac{1}{3}}\ln(|\mathcal{G}||\mathcal{H}|)^{\frac{1}{3}}$. With probability of at least $1-\delta$, the output predictor $f_B$ of Algorithm~\ref{alg:groupprepend} satisfies
    \[
    L(f_B\mid g)\leq \min_{h\in\mathcal{H}} L(h\mid g) + \mathcal{O}\!\left(\frac{\ln(|\mathcal{G}||\mathcal{H}|)^{\frac{1}{3}}\,n^{-\frac{1}{3}}}{\sqrt{P_n(g)}}\right),
    \quad \forall g\in\mathcal{G}.
    \]
\end{theorem}

\begin{proof}
For any $t=0,1,..,B-1,$
\begin{align*}
    L_n(f_{t+1})
    &= L_n(f_t) + P_n(g_{t+1})\bigl(L_n(f_{t+1}\mid g_{t+1}) - L_n(f_t\mid g_{t+1})\bigr)\\
    &= L_n(f_t) + P_n(g_{t+1})\bigl(L_n(h_{t+1}\mid g_{t+1}) - L_n(f_t\mid g_{t+1})\bigr)\\
    &\leq L_n(f_t) + \lambda.
\end{align*}
Therefore, the total number of updates is at most $1/\lambda$.

By the update rule, letting $f_B$ denote the final predictor, we have
\[
P_n(g)\Bigl(L_n(f_B\mid g)-\min_{h}L_n(h\mid g)\Bigr)\leq \lambda,\quad \forall g\in\mathcal{G}.
\]
Since the number of updates is bounded by $1/\lambda$, the number of possible outputs $f_B$ is at most
$|\mathcal{H}|^{1/\lambda}|\mathcal{G}|^{1/\lambda}$.
Thus, by Theorem~\ref{thm:Cond_Emp_Risk}, for any $g\in\mathcal{G}$,
\[
|L_n(f_B\mid g)-L(f_B\mid g)|
\leq 9\sqrt{\frac{2(1/\lambda+1)\ln(|\mathcal{G}||\mathcal{H}|)+\ln(8/\delta)}{nP_n(g)}}.
\]
Consequently,
\begin{align*}
    &L(f_B\mid g) - \min_{h\in\mathcal{H}}L(h\mid g)\\
    &\leq \bigl(L(f_B\mid g)-L_n(f_B\mid g)\bigr)
    + \bigl(L_n(f_B\mid g)-\min_{h\in\mathcal{H}}L_n(h\mid g)\bigr)+ \bigl(\min_{h\in\mathcal{H}}L_n(h\mid g)-\min_{h\in\mathcal{H}}L(h\mid g)\bigr)\\
    &\leq 9\sqrt{\frac{2(1/\lambda+1)\ln(|\mathcal{G}||\mathcal{H}|)+\ln(8/\delta)}{nP_n(g)}}
    + \frac{\lambda}{P_n(g)}
    + 9\sqrt{\frac{2\ln(|\mathcal{G}||\mathcal{H}|+\ln(8/\delta))}{nP_n(g)}}.
\end{align*}
Setting $\lambda=n^{-\frac{1}{3}}\ln(8|\mathcal{G}||\mathcal{H}|/\delta)^{\frac{1}{3}}$, we obtain
\begin{align*}
    L(f_B\mid g) - \min_{h\in\mathcal{H}}L(h\mid g)
    &\leq \frac{36\,\ln(8|\mathcal{G}||\mathcal{H}|/\delta)^{\frac{1}{3}}\,n^{-\frac{1}{3}}}{\sqrt{P_n(g)}}
    + \frac{n^{-\frac{1}{3}}}{P_n(g)}
    = \mathcal{O}\!\left(\frac{\ln(8|\mathcal{G}||\mathcal{H}|/\delta)^{\frac{1}{3}}\,n^{-\frac{1}{3}}}{\sqrt{P_n(g)}}\right).
\end{align*}
\end{proof}
\newpage
\section{Pseudocode for Fractional Variants of Prepend-Like Algorithms}
\begin{algorithm}[ht]
\caption{\texttt{Fractional Group Prepend}}
\label{alg:fractionalgroupprepend}
\begin{algorithmic}
\REQUIRE Groups $\mathcal{G}$, hypothesis class $\mathcal{H}$, i.i.d.\ examples from $\mathcal{D}$, error bound $\lambda$, step size $\eta$.
\STATE Compute $h_0 \in \arg\min_{h \in \mathcal{H}} L_n(h)$. Set $f_0 = [h_0] \in \mathrm{DL}_0[\mathcal{G}; \mathcal{H}]$
\FOR{$t = 0, 1, \ldots$}
    \STATE Define $f'_g(x) \gets f_t(x)+ \eta\, g(x)\bigl(h(x)-f_t(x)\bigr)$
        \STATE Compute
    \[
    g_{t+1} \in \arg\max_{g \in \mathcal{G}} P_n(g)(L_n(f_t \mid g) - L_n(f_g' \mid g))
    \]
    \IF{$P_n(g)\bigl(L_n(f_t \mid g) - L_n(f_{g_{t+1}}' \mid g)\bigr) \geq \lambda$}
    \STATE $f_{t+1}(x) \gets f_{g_{t+1}}'(x)$
    \ELSE
        \STATE Return $f_t$.
    \ENDIF
\ENDFOR
\end{algorithmic}
\end{algorithm}
\begin{algorithm}[ht]
\caption{\texttt{Fractional Prepend}}
\label{alg:fractionalprepend}
\begin{algorithmic}
\REQUIRE Groups $\mathcal{G}$, hypothesis class $\mathcal{H}$, i.i.d.\ examples from $\mathcal{D}$, error bound $\lambda$, step size $\eta$.
\STATE Compute $h_0 \in \arg\min_{h \in \mathcal{H}} L_n(h)$. Set $f_0 = [h_0] \in \mathrm{DL}_0[\mathcal{G}; \mathcal{H}]$
\FOR{$t = 0, 1, \ldots$}
    \STATE Define $f'_g(x) \gets f_t(x)+ \eta\, g(x)\bigl(h(x)-f_t(x)\bigr)$
        \STATE Compute
    \[
    g_{t+1} \in \arg\max_{g \in \mathcal{G}} L_n(f_t \mid g) - L_n(f_g' \mid g)
    \]
    \IF{$L_n(f_t \mid g) - L_n(f_{g_{t+1}}' \mid g) \geq \lambda$}
    \STATE $f_{t+1}(x) \gets f_{g_{t+1}}'(x)$
    \ELSE
        \STATE Return $f_t$.
    \ENDIF
\ENDFOR
\end{algorithmic}
\end{algorithm}
\end{document}